\documentclass[onefignum,onetabnum]{siamonline250211}
\usepackage{lipsum}
\usepackage{amsfonts}
\usepackage{graphicx}
\usepackage{epstopdf}
\usepackage{booktabs}
\usepackage{multirow}
\usepackage{subcaption}
\ifpdf
  \DeclareGraphicsExtensions{.eps,.pdf,.png,.jpg}
\else
  \DeclareGraphicsExtensions{.eps}
\fi

\usepackage{enumitem}
\setlist[enumerate]{leftmargin=.5in}
\setlist[itemize]{leftmargin=.5in}

\newsiamremark{remark}{Remark}
\newsiamremark{hypothesis}{Hypothesis}
\crefname{hypothesis}{Hypothesis}{Hypotheses}
\newsiamthm{claim}{Claim}
\newsiamremark{fact}{Fact}
\crefname{fact}{Fact}{Facts}
\newtheorem{dataassumption}[theorem]{Data Assumption}

\headers{Formulation and Learning Of The Forecasting Problem}{E. Calvello, E. Carlson, N. Kovachki, M.N. Manta and A.M. Stuart}

\title{Formulation and Learning Of The Forecasting Problem}%\thanks{Submitted to the editors DATE.
\author{Edoardo Calvello\thanks{California Institute of Technology, Pasadena, CA 
  (\email{ecalvello@lbl.gov}).}
\and Elizabeth Carlson\thanks{Oregon State University, Corvallis, OR 
  (\email{carleliz@oregonstate.edu}).}
\and Nikola Kovachki\footnotemark[3]\thanks{NVIDIA Corporation, Santa Clara, CA 
  (\email{nkovachki@nvidia.com}).}
\and Michael N. Manta\thanks{Stanford University, Stanford, CA 
  (\email{mnmanta@stanford.edu}).}
\and \newline Andrew M. Stuart\thanks{California Institute of Technology, Pasadena, CA 
  (\email{astuart@caltech.edu}).}
  }

\usepackage{amsopn}

\usepackage{bbm}
\usepackage{stmaryrd}
\usepackage[table,svgnames]{xcolor} % Required for \rowcolor
\usepackage{hhline}
\newcommand{\R}{\mathbbm{R}}
\newcommand{\Z}{\mathbbm{Z}}

\newcommand{\N}{\mathbbm{N}}

\newcommand{\cU}{\mathcal{U}}
\newcommand{\cV}{\mathcal{V}}
\newcommand{\cZ}{\mathcal{Z}}
\newcommand{\cG}{\mathcal{G}}

\newcommand{\cL}{\mathcal{L}}

\newcommand{\cW}{\mathcal{W}}

\newcommand{\sC}{\mathsf{C}}

\newcommand{\sL}{\mathsf{L}}
\newcommand{\sT}{\mathsf{T}}
\newcommand{\sF}{\mathsf{F}}
\newcommand{\sE}{\mathsf{E}}

\newcommand{\rd}{{\mathrm d}}
\newcommand{\dK}{d_\textrm{\scalebox{.7}{$K$}}}

\newcommand{\bbR}{\mathbb{R}}

\newcommand{\bbE}{\mathbb{E}}

\newcommand{\placeholder}{\mathord{\color{black!33}\bullet}}
\newcommand{\placeholderB}{\mathord{\color{black}\bullet}}

\usepackage{amsmath}  % for \DeclareMathOperator macro

\usepackage{algpseudocode}% http://ctan.org/pkg/algorithmicx

\usepackage{amsmath}
\usepackage{amssymb}
\usepackage{amsfonts}
\usepackage{dsfont}
\usepackage{mathtools}

\usepackage{amssymb}

\usepackage{microtype}

\usepackage[table,svgnames]{xcolor}         % colorse1

\usepackage{soul}

\definecolor{darkred}{rgb}{.7,0,0}
\definecolor{grey}{rgb}{.7,.6,.5}

\definecolor{darkgreen}{rgb}{0.05, 0.5, 0.06}

\definecolor{darkred}{rgb}{0.5, 0.05, 0.06}

\usepackage[normalem]{ulem}
\newcommand\mlremove{\bgroup\markoverwith
{\textcolor{red}{\rule[.5ex]{2pt}{1pt}}}\ULon}
\newcommand\asremove{\bgroup\markoverwith
{\textcolor{cyan}{\rule[.5ex]{2pt}{1pt}}}\ULon}
\newcommand\agreeremove{\bgroup\markoverwith
{\textcolor{blue}{\rule[.5ex]{2pt}{1pt}}}\ULon}

\usepackage{amsmath,amsbsy,amsgen,amscd,amsfonts} %amsthm
\usepackage{mathrsfs} % using this for \mathscr{g}
\usepackage{bbold}
\usepackage{csquotes}

\providecommand{\mathbbm}{\mathbb} % In case we don't load bbm

\newcommand{\sD}{\mathsf{D}}
\newcommand{\fp}{\mathfrak{p}}
\newcommand{\fq}{\mathfrak{q}}

\newcommand{\fx}{\mathfrak{x}}

\newtheorem{assumption}{Assumption}[section]
\newtheorem{example}{Example}[section]

\newlist{contributions}{enumerate}{5}
\setlist[contributions]{font={\bfseries}}
\crefname{contributionsi}{Contribution}{Contributions}

\newlist{probenum}{enumerate}{5}
\setlist[probenum]{font={\bfseries}}
\crefname{probenumi}{Problem}{Data Assimilation Problems}

\newcommand{\Psid}{\Psi^{\dagger}}

\newcommand{\Psimd}{\Psi^{\dagger}_{\mathsf{path}}}

\ifpdf
\hypersetup{
  pdftitle={An Example Article},
  pdfauthor={D. Doe, P. T. Frank, and J. E. Smith}
}
\fi

\begin{document}

\maketitle
% REQUIRED
\begin{abstract}
Machine learning has opened new frontiers in purely data-driven 
forecasting of dynamical systems; the resulting methods are showing 
considerable promise in some application domains.  However, in contrast 
to model-driven algorithms, analysis of these data-driven methods is 
poorly developed. In this paper we address this issue, developing a 
theory to underpin data-driven methods to solve forecasting problems. 
The theoretical framework comprises two key components: (i) establishing
the existence of a mapping defining the forecasting problem; (ii) 
showing that it is possible to approximate this mapping by supervised learning
techniques, to arbitrary accuracy. We work in the continuous-time 
and noise-free setting as this defines a simple canonical setting
in which to establish an underpinning theory; the continuous-time
setting leads to an operator formulation of the problem in which
the underlying forecasting map is defined on pathspace.
As a byproduct of our framework we show that analogous existence 
and approximation results may be obtained for the smoothing problem 
of recovering unobserved components of dynamical systems, from observed
components, using supervised learning. 
The theoretical results, for both forecasting 
and smoothing, are illustrated with experiments studying the 
Lorenz `63, Lorenz `96 and Kuramoto-Sivashinsky dynamical systems.
\end{abstract}

% REQUIRED

% REQUIRED
%\begin{MSCcodes}

%\end{MSCcodes}

\section{Introduction}
\label{sec:intro}

Machine learning has opened new frontiers in the forecasting 
of dynamical systems, using purely data-driven methods. The goal
of this paper is to establish a mathematical framework for the study
of data-driven forecasting. To this end, we consider the general 
dynamical system given by
\begin{subequations}
\label{eq:dynamical_system_intro}
\begin{align}
        \label{eq:dynsys_f}
        \dot{p} &= f(p, q), \qquad p(0) = p_0 \in \R^{d_p}, \\
        \label{eq:dynsys_g}
        \dot{q} &= g(p, q), \qquad q(0) = q_0 \in \R^{d_q}.
\end{align}
\end{subequations}
Our focus is on two problems: (i) existence of map $\Psid_F$ 
from observed $\{p(t)\}_{t\in [0,T]}$ to unobserved 
$\{p(t)\}_{t\in [T,T+\tau]}$;
(ii) approximation, using supervised learning, of $\Psid_F.$ 

The dynamical system \eqref{eq:dynamical_system_intro} will be assumed to
exist in order to address (i). However we assume no knowledge 
of the vector fields $(f,g)$, or even the dimension of $q$, in addressing
(ii) -- all data for the supervised learning problem is in the $p$ 
variable alone. Once the data-driven methods are trained they
simply take an instance of $\{p(t)\}_{t\in [0,T]}$ as input.

Note that once existence of $\Psid_F$ is established, map
$\Psimd: C([0,T];\R^{d_p}) \to C([0,T];\R^{d_p})$, which
maps $\{p(t)\}_{t\in [0,T]}$ to $\{p(t)\}_{t\in [\tau,T+\tau]}$,
may be defined. This allows us to define a dynamical system on $C([0,T];\R^{d_p})$ 
that may be approximated and rolled-out for long-term forecasting:
$$p_{j+1}=\Psimd(p_j)$$
where
\begin{align*}
 \Psimd(p)(t) & =p(t+\tau), \quad t \in [0,T-\tau),\\
 \Psimd(p)(t) & = \Psid_F(p)(t+\tau), \quad t \in [T-\tau,T].
\end{align*}

\subsection{Context and Literature Review}
\label{sec:CLR}

In the context of \eqref{eq:dynamical_system_intro} there are three problem classes which naturally arise: 
(I) \textit{forecasting}, which concerns estimation 
of $\{p(t)\}_{t\in [T,T+\tau]}$ from observed $\{p(t)\}_{t\in [0,T]}$;
(II) \textit{smoothing}, which concerns the estimation of unobserved 
$\{q(t)\}_{t\in [0,T]}$ from observed $\{p(t)\}_{t\in [0,T]}$; 
(III) \textit{filtering}, which concerns the online estimation of 
an underlying state, sequentially as observations are received.
Our focus in this paper is on (I); however the ideas are naturally
adapted to (II) and so we will also study smoothing in this paper. 
We will not address filtering, a subject which requires 
additional ideas.

The field of data assimilation (DA) is concerned with the use of 
observational data to solve either the smoothing or filtering
problems; either can be used as the basis of forecasting. This approach
has been adopted for several decades in the field of numerical weather 
prediction. Traditional model-driven techniques combine the observations with knowledge of the underlying dynamical
system to address the three tasks (I--III) \cite{kalnay2002atmoshperic,asch2016data,Evensenetal2022}.
Although primarily developed for geophysical applications in the ocean-atmosphere sciences,
such model-driven methods to address tasks (I--III) have been systematized and now constitute general-purpose methodologies
\cite{reich2015probabilistic,pandya2022review,sanz2023inverse}. In the last few years new purely data-driven algorithms, which 
do not require knowledge of the dynamical system when deployed, have started to emerge \cite{kurth2023fourcastnet,bi2023accurate,price2025probabilistic,Allen2025,alexe2024graphdop,kossaifi2026demystifying}. These methods have also been developed primarily in weather
forecasting, but present opportunities for deployment in many other domains. Developing
mathematical theory pertinent to these methods can play an important role
in the process of making the methods more widely applicable and is the focus of this paper.

Over the last half century, model-driven filtering, smoothing and forecasting methods have dominated 
in most application domains, starting from the seminal work of Kalman and Bucy for linear Gaussian
systems \cite{kalman1960new,kalman1961new}; the field evolved with introduction of the bootstrap, extended (ExKF) and ensemble Kalman filters (EnKF) \cite{doucet2000sequential,jazwinski2007stochastic,evensen1994sequential} which, respectively,
typically work well for small-, medium- and large-scale dynamical systems: bootstrap particle filter weights collapse in high dimensions and so they are best in low dimensions; EnKF has equal weights and has performed well for weather forecasting with state spaces in the billions of variables; ExKF falls between these cases as it makes a Gaussian approximation (mitigating weight collapse), but the covariances cannot be propagated in high dimension as they are too large. All model-driven approaches require knowledge of the dynamics, requiring
computationally expensive evaluations in the context of many large-scale applications of interest. Purely data-driven, model-free, methods for forecasting were proposed by Lorenz in 1969 \cite{lorenz1969atmospheric}, and go by the name of analog forecasting;
these methods were not widely adopted, however, primarily because of the lack 
of data to support them, and because they are discontinuous as a function of input data.
In the last decade, kernel analog forecasting
has been developed \cite{alexander2020operator}, a smoothening of the original
approach of Lorenz, supported by theory. Furthermore, methods such as dynamic
mode decomposition \cite{schmid2010dmd}, which are also purely data-driven, 
have been widely adopted and come with a developing theory \cite{mezic2020spectrum}. These methods have not yet, however, found success in the context of large-scale data assimilation and forecasting problems; their success
is mainly in the identification of large-scale coherent features in high-dimensional systems.

Machine learning has recently emerged as a novel computational tool for 
improving data assimilation and forecasting in both the model- and data-driven settings.
Learning has been introduced within model-driven DA in a variety of ways,
primarily in the filtering context. For example, it has been used for model error correction \cite{Farchi2021,levine2022}, for constructing cheap surrogates of the dynamics \cite{sanzalonso2025longtime, adrian2025data} and more recently to approximate conditioning in the assimilation step \cite{Bach2025Learning}. A data-driven approach to smoothing has recently been employed in \cite{Allen2025, gupta2026healda} for estimating the initial conditions of weather dynamics. Machine learning has also introduced a new approach to forecasting which is fully data-driven; this involves the learning of mappings which output predictions for the forecasted components of the system from their past observations. This approach has garnered wide interest in domains like weather forecasting, where both model evaluations and assimilation steps are particularly costly \cite{kurth2023fourcastnet,bi2023accurate, price2025probabilistic, alexe2024graphdop, kossaifi2026demystifying}. Time-series forecasting for
commercial and economic purposes has also seen considerable impact from machine-learning
\cite{wen2017multi,eisenach2020mqtransformer,challu2023nhits,liu2023itransformer,ansari2024chronos}. In a similar context, deploying pretrained time-series foundation models for forecasting of partially observed chaotic dynamical systems has been numerically investigated in \cite{zhang2025zeroshot}. Underpinning these data-driven approaches
with theoretical foundations is the goal of this paper.
In \Cref{tab:da_comparison} we summarize the distinction between model-driven DA, that has dominated
for the last half century, and recently emerging data-driven approaches including, but not limited to, the specific methods we study in this paper.

\begin{table}[htbp]
    \centering
    \footnotesize
    \caption{Comparison between model-driven and data-driven approaches.}
    \label{tab:da_comparison}
    \begin{tabular}{lcc}
        \toprule
        & \textbf{Model-driven} & \textbf{Data-driven} \\
        \midrule
        \textbf{Algorithm} 
        & Uses dynamical physics model
        & Uses data from dynamical model \\ & & and/or direct observations \vspace{1ex}\\
        \textbf{Limitation} 
        & Possibly expensive to evaluate and
        & Requires existence of direct mapping, \\
        & may miss some dynamics due to & a sufficiently dense training dataset and may not  \\
        & unmodeled physics & enforce explicit physical constraints \vspace{1ex}\\
        \textbf{Advantage} 
        & Interpretable and output satisfies 
        & Possibly cheap to evaluate and \\
        & prescribed physical constraints & agnostic to model dynamics \\ 
        \bottomrule
    \end{tabular}
\end{table}

\iffalse
\begin{table}[htbp]
    \centering
    \footnotesize
    \caption{Comparison between model-driven and data-driven approaches.}
    \label{tab:da_comparison}
    \begin{tabular}{lcc}
        \toprule
        & \textbf{Model-driven} & \textbf{Data-driven} \\
        \midrule
        \textbf{Algorithm} 
        & Uses dynamical model
        & Does not use dynamical model \\
        \textbf{Limitation} 
        & Possibly expensive model evaluations
        & Requires existence of direct mapping \\
        \textbf{Advantage} 
        & Interpretable 
        & Agnostic to dynamics \\
        \bottomrule
    \end{tabular}
\end{table}
\fi 

\subsection{Contributions and Outline}
\label{subsec:contributions}

In this paper we develop theory and algorithms for
data-driven forecasting; as a byproduct we also develop theory
and algorithms for smoothing, since the framework and machinery
we introduce is readily adapted to this problem. We work in the 
continuous-time and noise-free setting in order to establish
a clean framework in a canonical setting.
We prove existence of maps $\Psi_F^\dagger$ and $\Psi_S^\dagger$ defining
the forecasting and smoothing problems respectively,
under an observability-rank condition on the dynamics. This condition 
is intimately connected to the setup of \cite{hermann1977nonlinear}, 
a key work in nonlinear control theory which
establishes that, under a more general observability-rank 
condition which subsumes ours as a special case, 
distinct observations correspond to distinct dynamics. 
The existence of the desired operators allows approximation via 
neural operators, for which an emerging universal approximation 
theory has been developed \cite{kovachki2024operator}. 
In particular, we implement, and prove desirable universal 
approximation results for, the transformer neural operator 
\cite{calvello2024continuum}. The architecture underlying this 
neural operator is based on the attention mechanism, and is hence
well-adapted to picking up the long-range temporal correlations that
help solve smoothing and forecasting problems; we implement these neural operators, demonstrating the ability of the architecture to exploit such correlation structure in data. However, we emphasize that the goal of this paper is not
to advocate for a specific method to solve the forecasting problem;
rather it is to show that the forecasting problem can be framed
as a well-posed operator learning problem that is amenable to approximation.
Other methods can be evaluated within our framework; and our framework can,
in future work, be generalized to include discrete-time, noise and other generalizations.

Our contributions are to show the existence of operators on Banach 
space defining the forecasting and smoothing problems; to show that 
these maps can be approximated to arbitrary accuracy; and to study 
the supervised learning, and deployment, of these approximations.

\vspace{1ex}
\begin{contributions}[label=(C\arabic*)]
\item 
\label{contrib:control}
We bridge the control theoretic literature with data-driven forecasting 
and smoothing by using an observability condition on the dynamics, which guarantees existence of a locally-defined continuous operator mapping an observed component of the dynamical system to an unobserved component; this condition is defined in \Cref{assump:observability}. Under this condition we establish,
in Theorems \ref{prop:psiF} and \ref{prop:psiS} respectively, 
the existence and regularity of the forecasting and smoothing maps
$\Psid_F$ and $\Psid_S$, respectively.
\vspace{1ex}
\item \label{contrib:approx_smoothing} We prove a pair of theorems establishing that a class of transformer-based neural operators are capable of 
approximating, to arbitrary accuracy, 
$\Psid_F$ and $\Psid_S$: see \Cref{thm:UAforecasting}
and  \Cref{thm:UAsmoothing}, respectively.

\item \label{contrib:numerics}
We showcase the universal approximation theory by 
learning and deploying transformer 
neural operators, which have the ability to detect and exploit long-range temporal correlations, on smoothing and forecasting problems; we study the Lorenz `63, Lorenz `96, and Kuramoto-Sivashinsky dynamical systems.
\end{contributions}
\vspace{1ex}

In Subsection \ref{ssec:key} we describe key pre-existing methodologies
that tackle problems closely related to those we study here. Then,
after introducing notation that will be used throughout the paper in \Cref{subsec:notation}, we describe the general dynamical system setting in \Cref{subsec:setup}; here we also define the smoothing and forecasting problems which will be the object of investigation. 

In \Cref{sec:Observability} we introduce the observability framework
that underpins our analysis. Section \ref{sec:force} is concerned
with existence of the forecasting map $\Psid_F$, and Section \ref{sec:forca}
with approximation of this map by the transformer neural operator.
In Section \ref{sec:smooth} we generalize this existence and approximability
theory to the smoothing map $\Psid_S.$
In \Cref{sec:experiments} we demonstrate the theory numerically, 
including discussion of the robustness of our methodology to 
observational noise. 
Finally in \Cref{sec:conclusions} we discuss the results of the paper and conclude by outlining avenues for further work, including the discrete-time
setting and the effect of noise.

\subsection{Key Pre-Existing Methodologies}
\label{ssec:key} 

Our goal is to establish a framework for the forecasting problem,
not to provide a state-of-the-art algorithm. In this context it is
of value to mention pre-existing methodologies used to study
forecasting, smoothing and filtering.

In the context of forecasting we first highlight the rich Koopman operator 
framework, which delivers both methodology and theory for approximating 
dynamical systems from data \cite{mauroy2020koopman}. In particular the work 
of Giannakis and co-workers has developed a theory-backed approach to 
data-driven forecasting \cite{giannakis2019data,alexander2020operator};
subsequent work \cite{giannakis2023learning} has extended the Koopman
based methodology to an efficient streaming formulation.
In our notation, the forecasting problem addressed in 
\cite{alexander2020operator} is to consider 
equation \eqref{eq:dynamical_system_intro} to be initialized at a draw from 
an (assumed to exist) invariant measure and to compute the conditional expectation
of $(p,q)$, at time $T+\tau$, given the observation of $p$ on $t \in [0,T].$ 
We note that the \emph{observability} assumptions that underpin our
work are not made in the papers 
\cite{giannakis2019data,alexander2020operator}; it is for this reason
that the forecast takes the form of a conditional expectation. Furthermore,
our approach leads to a semigroup structure for prediction on the space
of paths, via the map $\Psimd;$ the approach of Giannakis and co-workers
leads to predictions at fixed lead times which cannot be meaningfully composed.

Takens' delay embedding theorem \cite{takens2006detecting,sauer1991embedology}
is closely related to the approach to forecasting that we adopt here. 
It is formulated for discrete-time dynamical systems. If we sample 
the continuous-time solution operator defined by 
\eqref{eq:dynamical_system_intro} at rate $\tau$ then we obtain a
discrete-time dynamical system on $\R^{d_p+d_q}$ and we may consider 
the resulting one-step ahead forecasting problem for $p$ in $\R^{d_p}$, 
using information from the history of $p$.  The Takens methodology 
leads to a dynamical system on $\R^{n d_p}$, a discrete analog of $\Psimd$
in continuous-time; here $n$ is
the number of delays needed to reconstruct the dynamics. Whilst such an integer $n$ exists, generically, in practice it is not known and must be determined experimentally; furthermore the discrete-time dynamical system does not interface so cleanly with the continuous time perspective we assume here; in contrast, our proposed
method provably reconstructs the forecasting dynamics, without the need for  tunable parameter
$n$, and the concrete control-theoretic perspective of \cite{hermann1977nonlinear} provides a natural basis for the continuous time theory and algorithms that we develop here.

For the smoothing problem our training data comprises pairs of trajectories $(p,q)$ over time-interval
$[0,T].$ This data could also be used to learn approximation of the vector fields $(f,g)$ using, for example, the SINDy approach~\cite{brunton2016discovering}. Once this dynamical system is learned, data assimilation can be applied to solve filtering and smoothing problems as in, for example, \cite{gottwald2021combining} and \cite{gottwald2021supervised}. We have proposed a different end-to-end approach which directly learns solution of the smoothing problem and bypasses learning the dynamical system itself. However, knowledge of the dynamical system may be useful for other downstream tasks, or for interpretability, and in such situations pre-existing methodologies have a clear advantage.

Another set of ideas, aimed at recovering $q$ from $p$ in a filtering context, rather than smoothing, is the idea of synchronization \cite{pecora1990synchronization} and the related concept of determining modes \cite{olson2003determining}, published independently. Both papers exploit an
idea present in the celebrated paper of Foias and Prodi \cite{foias1967comportement} namely that, in certain dynamical
systems, solution of equation \eqref{eq:dynsys_g}, for $q$ with observation $p$ given, and with $q(0) \ne q_0,$
will synchronize with solution of equation \eqref{eq:dynsys_g}, 
for $q$ with observation $p$ again given and with $q(0) = q_0$. This idea was developed as a form of filtering in data assimilation in fluid
mechanics in \cite{henshaw2003numerical}, studied in continuous time in
\cite{olson2003determining} and then in discrete time in \cite{hayden2011discrete}.

\subsection{Notation}
\label{subsec:notation}

Throughout we denote the positive integers and non-negative integers
respectively by $\N=\{1,2,\cdots\}$ and $\Z^+=\{0,1,2,\cdots\},$
and the notation $\R=(-\infty,\infty)$ and $\R^+=[0,\infty)$ for the
reals and the non-negative reals. For a set $D \subset \R^m$, we denote 
by $\Bar{D}$ the closure. 
Given two vector spaces $\cU,\cV$ we denote by $\sL(\cU;\cV)$ the space of linear operators acting between $\cU$ and $\cV$. For a compact set $D\subset\R^m$, we denote by $C(D;\R^d)$ the infinite dimensional Banach space of continuous functions mapping the set $D$ to the $d$-dimensional vector space $\R^d$. The space is endowed with the supremum norm $\|\placeholder\|_\infty$. We will sometimes use the shorthand notation $C(D)$ to denote continuous functions defined on the domain $D$,
whenever the image space is unambiguous or not germaine. Similarly, we sometimes drop the domain notation completely, when it is clear from context. For $s \in \Z^+$ we denote by $C^s$ the space of continuously differentiable functions up to order $s$. Let $C^{\infty}$ denote the infinite dimensional real vector space of all infinitely differentiable functions. For any function $f\in C^s\bigl(\R^{d_1};\R^{d_2}\bigr)$ and any $0\leq k\leq s$, we define by $\sD^k f(v)$ the $k$th derivative at $v\in \R^{d_1}$ viewed as the linear operator $\sL\bigl(\bigotimes_{i=1}^k\R^{d_1}; \R^{d_2}\bigr)$. We endow the space $C^s(D;\R^{d_2})$, for $D\subset \R^{d_1}$ compact, with the norm $\|\placeholder\|_{C^s}$ defined as
\begin{equation}
\label{eq:norm}
    \|f\|_{C^s} = \sum_{j=0}^s \| \sD^j f\|_\infty,
\end{equation}
for any function $f\in C^s(D;\R^{d_2})$. Given $s \in \N$, for any function $f\in C^s\bigl(\R^{d_1};\R^{d_2}\bigr)$ and a vector field $w\in C^{s-1}\bigl(\R^{d_1};\R^{d_1}\bigr)$, we also define $\mathcal{L}_w\,f(v)$, the Lie derivative\footnote{We use the terminology Lie derivative throughout the paper for derivatives in this sense,
viewing the Lie derivative as a mapping that takes one vector-valued function to another.} of $f$ along $w$, as
\[
\mathcal{L}_w\,f(v) = \sD f (v) w(v), \qquad v\in \R^{d_1}.
\]
Note that $\mathcal{L}_w\,f \in C^{s-1}\bigl(\R^{d_1};\R^{d_2}\bigr).$

For $k\leq s$ we also denote by $\cL^k_w \,f$ the Lie derivative of $f$ along $w$ of order $k$, defined by $k-$fold application of $\cL_w$. We will occasionally drop the vector field from the notation when it is clear 
from the context.

\subsection{Dynamical System Setup}
\label{subsec:setup}

Consider the dynamical system \eqref{eq:dynamical_system_intro}.
We make the following regularity assumption:

\begin{assumption}[Regularity]
\label{assump:regularity}
 Assume that, for some $k \in \N$,
 \begin{itemize}
     \item $f \in C^k (\R^{d_p + d_q}; \R^{d_p})$;
     \item $g \in C^k (\R^{d_p + d_q}; \R^{d_q})$.
 \end{itemize}
Assume, furthermore, that solutions to \eqref{eq:dynamical_system_intro} exist for all $t \in \R^+.$
\end{assumption}

Under \Cref{assump:regularity} it follows that, for any $T>0$, $p \in C^k ([0,T];\R^{d_p})$ and $q \in C^k ([0,T];\R^{d_q})$. We define by $\Phi:\R^+ \times\R^{d_p+d_q}\to\R^{d_p+d_q}$ the semigroup of operators associated to the dynamical system, so that $\bigl(p(t),q(t)\bigr)=\Phi(t,p_0,q_0).$  
Let $I \subset \R^{d_p + d_q}$ denote a compact set of initial conditions to \eqref{eq:dynamical_system_intro} and fix $T>0$. Since by \Cref{assump:regularity} solutions to \eqref{eq:dynamical_system_intro} exist for all $t \in \R^+$, the result of \cite[Theorem 2.10]{teschl2012ordinary} implies that $\Phi(\placeholder,\placeholder)\in C^k([0,T]\times I;\R^{d_p + d_q})$; therefore, it may be deduced that the map $x\mapsto \Phi(\placeholder,x)\in C^k([0,T];\R^{d_p + d_q})$ is continuous. Since the image of a compact set under a continuous map is compact, we deduce that the set of orbits 
$S^I_{[0,T]} = \{\Phi(t,x)\,\text{for }t\in [0,T]: x \in I \}$ is compact in
$C^k([0,T];\R^{d_p + d_q})$. Let $\pi_p \colon \R^{d_p + d_q} \to \R^{d_p}$ be the projection map onto the first $d_p$ coordinates and $\pi_q \colon \R^{d_p + d_q} \to \R^{d_q}$ be the projection map onto the remaining $d_q$ coordinates of the state space. We note that these are both continuous functions. Define the projected sets
\begin{equation}
\label{eq:Ssets}
S^I_{[0,T],p} = \{\pi_p (s): s \in S^I_{[0,T]}\}, \quad S^I_{[0,T],q} = \{\pi_q (s): s \in S^I_{[0,T]}\};
\end{equation}
these too are compact in $C^k([0,T];\R^{d_p})$ and $C^k([0,T];\R^{d_q})$ since $S^I_{[0,T]}$ is compact and projection is continuous. Fixing some $\tau>0$, for the set of orbits defined on the future time interval $S^I_{[T,T+\tau]} = \{\Phi(t,x)\,\text{for }t\in [T,T+\tau]: x \in I \}$ we similarly define the projected set
\begin{equation}
    \label{eq:forecastingSet}
    S^I_{[T,T+\tau],p} = \{\pi_p (s): s \in S^I_{[T,T+\tau]}\}.
\end{equation}
It is readily deduced that $S^I_{[T,T+\tau],p}$ is also compact in $C^k([T, T+\tau];\R^{d_p})$. 
Given the dynamical systems \eqref{eq:dynamical_system_intro} equipped with the regularity assumption in \Cref{assump:regularity}, we define the two problems central to this paper. We assume that
pair $(p,q)$ solves \eqref{eq:dynamical_system_intro} but that vector field $(f,g)$ is not known to us. We then consider the following problems.
\vspace{1ex}
\begin{probenum}[label=(P\arabic*)]

\item \label{prob:forecasting} \textit{Forecasting}: predict
$p \in S^I_{[T,T+\tau],p}$ from $p \in S^I_{[0,T],p}.$ 
\vspace{1ex}
\item \label{prob:smoothing} \textit{Smoothing}: recover $q\in S^I_{[0,T],q}\subset C^k ([0,T];\R^{d_q})$ from observed $p\in S^I_{[0,T],p} \subset C^k ([0,T];\R^{d_p})$. % for $(p, q) \in S^I_{[0,T]}$.

\end{probenum}

\section{Observability}
\label{sec:Observability}
%\label{sec:UAforDA}

Assume that, given $\{p(t)\}_{t\in [0,T]}$, a $p-$trajectory generated
by \eqref{eq:dynamical_system_intro}, it is possible to
find $\bigl(t^\star, q(t^\star)\bigr)$, a point on the corresponding
$q-$trajectory. Then, given that dynamical system defined by 
\eqref{eq:dynamical_system_intro} is assumed to exist, the governing
equations may in principle be integrated
forward and backward from $\bigl(t^\star, p(t^\star), q(t^\star)\bigr)$ to 
define both the forecasting and smoothing maps $\Psid_F, \Psid_S.$
This section is hence focused on the question of 
recovering $q(t)$, at some given time $t$,
from $p$ and its derivatives at that time, noting 
that these derivatives are determined by
$\{p(t)\}_{t\in [0,T]}$, assuming sufficient differentiability.

In \Cref{ssec:noto} we introduce the notation required to discuss the observability-rank condition
which encapsulates solvability for $q(t)$ at some time $t$.  The observability-rank condition itself is
introduced in \Cref{ssec:obs}. We proceed in \Cref{ex:Lorenz} with an example of
observability in the Lorenz `63 equation.
\Cref{rem:global} concludes with some remarks 
concerning the local nature
of our construction.

\subsection{Observability Setup}
\label{ssec:noto}

To state the desired observability condition we introduce the following notation relating to the dynamical system. Let $\cL^m$ denote the $m^{th}$ Lie derivative of $f$ along the vector field $(f,g)$ and let $\widetilde{F}^{(n)}\colon\R^{d_p+d_q}\to\R^{nd_p}$ be defined from the first $n$ Lie derivatives of $f$ along the vector field $(f,g)$ as follows:
    \begin{equation}
    \label{eq:Lie_tilde}
    \widetilde{F}^{(n)}(\mathfrak{p},\mathfrak{q}) = \Bigl( \cL^0\,f(\mathfrak{p},\mathfrak{q})^\top,\cL^1\,f(\mathfrak{p},\mathfrak{q})^\top,\ldots, \cL^{n-1}\,f(\mathfrak{p},\mathfrak{q})^\top\Bigr)^\top,
    \end{equation}
    for any $(\fp, \fq) \in \R^{d_p + d_q}$.
We also define $F^{(n)}\colon\R^{d_p+d_q}\to\R^{(n+1)d_p}$ by
    \begin{equation}
    \label{eq:Lie}
    F^{(n)}(\mathfrak{p},\mathfrak{q}) = \Bigl(\fp^{\top},\, \cL^0\,f(\mathfrak{p},\mathfrak{q})^\top,\cL^1\,f(\mathfrak{p},\mathfrak{q})^\top,\ldots, \cL^{n-1}\,f(\mathfrak{p},\mathfrak{q})^\top\Bigr)^\top,
    \end{equation}
    for any $(\fp, \fq) \in \R^{d_p + d_q}$. We note that the two functions $\widetilde{F}^{(n)}$ and
    $F^{(n)}$ differ only through the inclusion of $\fp\in \R^{d_p}$ in $F^{(n)}.$ Using either we can write down equations which, potentially, enable the extraction of $q(t)$ from knowledge of $p(\cdot)$ and its derivatives at $t.$

    To this end we show how a collection of Lie derivatives of the vector field $f$ can be expressed in terms of derivatives of $p.$ By considering \eqref{eq:dynsys_f}, we may deduce that for any $(p,q) \in S^I_{[0,T]},$ it holds that
    \begin{align*}
    \cL\,f \bigl(p(t),q(t) \bigr) &= \partial_p f\bigl(p(t),q(t) \bigr)\cdot f \bigl(p(t),q(t) \bigr) + \partial_q f\bigl(p(t),q(t) \bigr)\cdot g \bigl(p(t),q(t) \bigr)\\
    & = \partial_t^2\, p(t),
    \end{align*}
    for all $t\in[0,T]$.
    By repeated application of Lie derivatives, it holds for any $(p, q)\in S^I_{[0,T]} \subset C^k([0,T];\R^{d_p})\times C^k([0,T];\R^{d_q})$ and for all $n \leq k-1$ that
    \begin{equation}
    \label{eq:connection_control}
    \cL^{n}\,f \bigl(p(t),q(t) \bigr) = \partial_t^{n+1}\, p(t),
    \end{equation}
    for all $t\in [0,T]$. Define operator $\widetilde{P}^{(n)}\colon C^k([0,T];\R^{d_p})\to \prod_{j=1}^{n} C^{k-j}([0,T];\R^{d_p})$ via its action on $p\in C^k([0,T];\R^{d_p})$ as follows:
    \begin{equation}
    \label{eq:P_op_tilde}
        \bigl(\widetilde{P}^{(n)}(p)\bigr)(t) = \Bigl(\partial_t^1\,p(t)^\top,\ldots, \partial_t^n\,p(t)^\top\Bigr)^\top.
    \end{equation}
We may then deduce that any $(p, q)\in S^I_{[0,T]}$ satisfies the equation
    \begin{equation}
    \label{eq:observability_eq_tilde}
        \widetilde{F}^{(n)}\bigl(p(t), q(t) \bigr) = \bigl(\widetilde{P}^{(n)}(p)\bigr)(t),
    \end{equation}
    for all $t\in [0, T]$.
    Similarly, define operator $P^{(n)}\colon C^k([0,T];\R^{d_p})\to \prod_{j=0}^{n} C^{k-j}([0,T];\R^{d_p})$ 
    by
    \begin{equation}
    \label{eq:P_op}
        \bigl(P^{(n)}(p)\bigr)(t) = \Bigl(\partial_t^0\, p(t)^\top, \partial_t^1\,p(t)^\top,\ldots, \partial_t^n\,p(t)^\top\Bigr)^\top,
    \end{equation}
    to deduce that any $(p, q)\in S^I_{[0,T]}$ satisfies the equation
    \begin{equation}
    \label{eq:observability_eq}
        F^{(n)}\bigl(p(t), q(t) \bigr) = \bigl(P^{(n)}(p)\bigr)(t),
    \end{equation}
    for all $t\in [0, T]$. 
    Notice that in the topology of \eqref{eq:norm}, $P^{(n)}$ is continuous for all $n$ since 
    differentiation $(k-j)$-times is continuous in each 
    $C^{k-j}([0, T];\R^{d_p})$ component for $0 \leq j \leq n$. 
    Similar comments apply to $\widetilde{P}^{(n)}.$

\subsection{Observability-Rank Condition}
\label{ssec:obs}

    Either of the identities \eqref{eq:observability_eq_tilde} or \eqref{eq:observability_eq}  may be
    used to define an equation for determination of $q(t)$ from knowledge of $p(\cdot)$ and its derivatives at $t.$
    We first discuss this by using $\widetilde{P}^{(n)}$, which is perhaps the most direct way to consider the
    question of recovering $q(t);$ we then modify and frame in terms of $P^{(n)}$ enabling us to connect
    with the control theory literature.
    
    Fixing some $t\in[0,T]$, and choosing 
    $\widetilde{L}:\R^{nd_p}\to \R^{d_q}$, we consider the following equation for $\fq$:
    \begin{equation}
    \label{eq:observability_eq_inv3}
    \widetilde{L}\widetilde{F}^{(n)}\bigl(p(t), \mathfrak{q} \bigr) = \widetilde{L}\bigl(\widetilde{P}^{(n)}(p)(t)\bigr).
    \end{equation}
    Given $p(\cdot)$ and its derivatives at $t$, 
    this constitutes a system of $d_q$ equations for $\fq \in \R^{d_q}$. 
     If there is a time $t$ and linear operator $\widetilde{L}$ such that the operator $\widetilde{L}\widetilde{F}^{(n)}(p(t),\,\placeholder):\R^{d_q}\to\R^{d_q}$ is invertible, then there exists a unique solution $\fq$ 
     to \eqref{eq:observability_eq_inv3} and, by \eqref{eq:observability_eq_tilde}, $\fq=q(t).$ 
     This then facilitates determination of $q(t)$ from knowledge of $p(\cdot)$ and its derivatives at $t.$

     A similar approach can be undertaken using $P^{(n)}$ rather than  $\widetilde{P}^{(n)}$.
     Again fixing some $t\in[0,T]$, and now choosing 
    ${L}:\R^{(n+1)d_p}\to \R^{d_p+d_q}$, we consider the equation
    \begin{equation}
    \label{eq:observability_eq_inv2}
    {L}{F}^{(n)}\bigl(\fp, \mathfrak{q} \bigr) = {L}\bigl({P}^{(n)}(p)(t)\bigr).
    \end{equation}
    This constitutes a system of $d_p+d_q$ equations for $(\fp,\fq) \in \R^{d_p+d_q}$.
     If there exists a time $t$ and linear transformation $L$ such that the operator ${L}{F}^{(n)}(\placeholder,\,\placeholder):\R^{d_p+d_q}\to\R^{d_p+d_q}$ is invertible, then there exists a unique solution $(\fp,\fq)$, to \eqref{eq:observability_eq_inv2} and, by \eqref{eq:observability_eq}, $(\fp,\fq)=\bigl(p(t),q(t)\bigr)$ and, in particular, we recover $q(t).$ 
     This motivates the following assumption.

\begin{assumption}[Observability-Rank Condition]
\label{assump:observability}
    We say that the system \eqref{eq:dynamical_system_intro} satisfies the observability-rank condition at $(\mathfrak{p}, \mathfrak{q})\in \R^{d_p + d_q}$, if there exists $n\in\N$ and a linear operator $L\colon\R^{(n+1)d_p}\to\R^{d_p+d_q}$ such that the rank of the matrix $\sD LF^{(n)}\bigl(\fp, \fq \bigr)$ is $d_p + d_q$.
\end{assumption}

    We recall that the inverse function theorem for Euclidean spaces states that for $G\colon \R^{d} \to \R^d$ a $C^1$ map, if $\sD G$ is invertible at a point $x \in \R^d$ then there exist open neighborhoods $U, V$ of $x$ and $G(x)$ respectively, such that there exists a continuous inverse $G^{-1}\colon V \to U$. The existence of open sets $U$ and $V$ along with continuous inverse is what we refer to as local invertibility. Thus:
    
    \begin{proposition}
Assume that the  observability-rank condition holds at $(\mathfrak{p}, \mathfrak{q})\in \R^{d_p + d_q}$. Then $LF^{(n)}(\placeholder, \placeholder)$ is locally invertible at $(\fp,\fq)$.
    \end{proposition}

We note that \Cref{assump:observability} is sufficient, but not necessary, for local invertibility of $LF^{(n)}$.

\begin{remark}[Observability in Control Theory]
    We note that the observability-rank condition from \Cref{assump:observability} is an instance of the observability condition from the seminal paper \cite{hermann1977nonlinear}. In particular, considering the observation function $h(\mathfrak{p},\mathfrak{q}) = \mathfrak{p}$ for any $\mathfrak{p}\in\R^{d_p}, \mathfrak{q}\in \R^{d_q}$, we may reformulate \eqref{eq:observability_eq} in terms of this observation function $h$ and hence \Cref{assump:observability} in the control theoretic setting of \cite{hermann1977nonlinear}. We elaborate on this connection in \Cref{appendix:observability_control}.
$\diamond$\end{remark}

\subsection{Observability of Lorenz `63 Model}
\label{ex:Lorenz}

Consider the Lorenz `63 dynamical system in the form
\begin{subequations}\label{eq:L63_dynamic}
\begin{align} 
\dot{x} &= \sigma (y - x), \\ 
\dot{y} &= x (\rho - z) - y, \\ 
\dot{z} &= xy - \beta z. 
\end{align} 
\end{subequations}
initialized at $(x_0,y_0,z_0)\in \R^3$. We consider the setting where the component $p = x$ is observed and $q = (y, z)$ is unobserved. It is readily shown that the system satisfies the observability-rank condition, given this choice
of observation. We note that by \eqref{eq:observability_eq} it holds that
\begin{align}
    &F^{(2)}(x_0, y_0, z_0)
    = \begin{pmatrix}
        x_0\\
        \sigma(y_0 - x_0)\\
        \sigma\Bigl(x_0(\rho - z_0) - y_0 - \sigma(y_0 - x_0)\Bigr)\\
    \end{pmatrix}.
\end{align}
Now, for $L$ the identity matrix, the associated differential is
\begin{align*}
    \sD\, LF^{(2)}(x_0, y_0, z_0) = \begin{pmatrix}
        1 & 0 & 0\\
        -\sigma & \sigma & 0\\
        \sigma (\rho + \sigma - z_0) & -\sigma (\sigma + 1) & -\sigma x_0
    \end{pmatrix}.
\end{align*}
From the upper-triangular structure it follows readily that this is of full rank for $x_0 \ne 0$. Hence we can conclude that the Lorenz `63 dynamical system satisfies the observability-rank condition at $t=0$, whenever $x_0\neq 0$.
In particular, the map $\{x(t)\}_{t\in [0,T]} \mapsto \{\bigl(y(t),z(t)\bigr)\}_{t\in [0,T]}$ is
well-defined whenever $x_0 \ne 0.$

To contrast with observing $\{x(t)\}_{t\in [0,T]}$, we note that the Lorenz `63 model is \textit{not} observable if the only data given is $\{z(t)\}_{t\in [0,T]}$. Specifically, if $(x,y,z)$ satisfies \eqref{eq:L63_dynamic}, then
$(-x,-y,z)$ is also a solution to \eqref{eq:L63_dynamic}. Thus, there cannot exist a well-defined mapping 
$\{z(t)\}_{t\in [0,T]} \mapsto \{\bigl(x(t),y(t)\bigr)\}_{t\in [0,T]}.$
We demonstrate this experimentally in \Cref{sec:experiments}. We further note that this example is consistent with the synchronization phenomena highlighted in \cite{pecora1990synchronization}. There the authors demonstrate that driving the Lorenz `63 system with known $x-$path, from a different simulation of the same system, makes the response subsystem for $(y,z)$ synchronize with $(y,z)$ coordinates of
the realization that gave rise to driving signal $x$. Synchronization does not
occur when driving the system with a  $z-$path and studying the $(x,y)$ subsystem.

\subsection{Local Versus Global Existence} \label{rem:global}

In the sections that follow we use the observability rank condition 
to prove existence of forecasting and smoothing operators 
$\Psid_F$ and $\Psid_S$, respectively. Since our observability-rank 
condition \Cref{assump:observability} is local, the operators constructed 
have a domain of input trajectories near to the point of invertibility 
$(\fp,\fq)$. The set of these trajectories is determined, non-constructively, 
by the inverse function theorem. It is of general interest to determine 
the largest possible domain of input trajectories for which
$\Psi^\dagger_F$, or $\Psid_S$, exists and is continuous. A potential approach to this problem is to determine all points of invertibility and stitch together a global inverse from the open sets given by the inverse function theorem. Such analysis, however, will likely need to be carried out on a case-by-case basis and is not the focus of the current work.
In the setting of the Lorenz `63 system studied in \Cref{ex:Lorenz}, we show that $U = (\mathbb{R}\setminus \{0\}) \times \mathbb{R}^2$. Therefore, for any compact set $I \subset U$ of initial conditions, the operators $\Psid_F$ and $\Psi^\dagger_S$ are defined and continuous on $S^I_{[0,T],p}$.

\section{Existence of the Forecasting Map} \label{sec:force}

The key result underlying the existence of the forecasting map
is the following observability property: 

\begin{proposition}[Existence of map $W:p\mapsto (p(0), q(0))$]
\label{prop:W}
Consider the dynamical system \eqref{eq:dynamical_system_intro} satisfying the regularity \Cref{assump:regularity}, for integer $k$, and the observability \Cref{assump:observability} at some $(\mathfrak{p},\mathfrak{q})\in\R^{d_p + d_q}$, for integer $n \le k.$ Let $U \ni (\fp, \fq)$ and $V \ni LF^{(n)}(\fp, \fq)$ be the open sets given by the inverse function theorem. Then, for every compact $I \subset U$, there exists a continuous map $W\colon S^I_{[0,T],p} \subset C^{k}([0,T];\R^{d_p}) \to \R^{d_p + d_q}$ such that, for every $(p, q) \in S^I_{[0,T]}$,
\begin{equation}
\label{eq:W}
W(p) = (p(0), q(0)).
\end{equation}
\end{proposition}
\begin{proof}
Observe that from \eqref{eq:observability_eq_inv2} for every $(p, q) \in S^I_{[0,T]}$ we have that 
\begin{equation}
\label{eq:W-capability}
    LF^{(n)}(p(0), q(0)) = L(P^{(n)}(p)(0))
\end{equation}
and $LF^{(n)}(\placeholder, \placeholder)$ is locally invertible by the inverse function theorem from \Cref{assump:observability}. Furthermore, the evaluation functional at time $t=0$, namely $\delta_0 \colon \prod_{j=0}^n C^{k-j}([0,T];\R^{d_p}) \to \R^{(n+1)d_p}$, is linear and bounded, hence continuous with respect to the topology of \eqref{eq:norm}. Using these operators, we can construct $W\colon S^I_{[0,T],p} \subset C^{k}([0,T];\R^{d_p}) \to \R^{d_p + d_q}$ so that
\begin{equation}
      W(p) = \bigl(LF^{(n)}\bigr)^{-1}
      \circ L \circ \delta_0 \circ P^{(n)}(p).
\end{equation}
From \eqref{eq:W-capability} we have that $ L \circ \delta_0 \circ P^{(n)}(p) \in V$ for all $p$, so composition with $(LF^{(n)})^{-1}$ is well-defined. By construction $(LF^{(n)})^{-1}$ is continuous and $W$ is continuous since it is a composition of continuous operators. For all $(p, q) \in S^I_{[0,T]}$ we have $W(p) = (p(0), q(0))$ since $(p(0), q(0))$ is a solution to \eqref{eq:W-capability} and $LF^{(n)}$ is a bijection from $U$ to $V$. 
\end{proof}

Using \Cref{prop:W} we establish the existence of an operator $\Psid_F$
mapping the observed portion of $p$ trajectories, $p |_{[0,T]}$, 
to the future portion $p |_{[T,T+\tau]}$.
This in turn defines maps $\Psimd: C([0,T];\R^{d_p}) 
\to C([0,T];\R^{d_p})$,
mapping $p |_{[0,T]}$ into 
$p |_{[\tau,T+\tau]},$ and defined in Section \ref{sec:intro}.
This map may then be iterated to roll-out forecasts arbitrarily far into
the future.

\begin{theorem}[Existence of map $\Psi_F^\dagger:p |_{[0,T]}\mapsto p |_{[T,T+\tau]}$]
\label{prop:psiF}
Consider the dynamical system \eqref{eq:dynamical_system_intro} satisfying 
regularity \Cref{assump:regularity}, for integer $k$, and the observability \Cref{assump:observability} at some $(\mathfrak{p},\mathfrak{q})\in\R^{d_p + d_q}$, for integer $n \le k.$
Let $ U \ni (\fp, \fq)$ and $V \ni LF^{(n)}(\fp, \fq)$ be the open sets given by the inverse function theorem. 
Then, for every compact $I \subset U$, there exists a continuous operator $\Psi_F^\dagger \colon S^I_{[0,T],p} \subset C^k([0,T];\R^{d_p}) \to S^I_{[T,T+\tau],p} \subset C^k([T, T+\tau]; \R^{d_p})$ such that, for every $(p, q) \in S^I_{[0,T]}$,
\begin{equation}
    \label{eq:p_to_q_map3}
    \Psi_F^\dagger(p)(t) = p(t), \qquad T\leq t\leq T+\tau.
\end{equation}
\end{theorem}
\begin{proof}
We construct such a $\Psi_F^\dagger$ using $W$ from \Cref{prop:W} and composition with the solution operator $\Phi$. Define $\Psi_F^\dagger \colon S^I_{[0,T],p} \subset C^k([0,T];\R^{d_p}) \to S^I_{[T,T+\tau],p} \subset C^k([T, T+\tau]; \R^{d_p})$ by 
\begin{equation}
    \label{eq:p_to_q_map4}
    \Psi_F^\dagger(p)(t) = \pi_p \circ \Phi(t, W(p)), \quad T\leq t\leq T+\tau.
\end{equation}
We first recall that $\pi_p$ is continuous. By construction it holds that $W(p) = (p(0), q(0))$, hence $\pi_p \circ \Phi(t, W(p)) = {p}(t)$ for $T\leq t\leq T+\tau$. We observe that this is a continuous operator since $\Phi$ is continuous with respect to initial conditions, $W$ is continuous with respect to $p$ and composition of continuous operators is continuous.
\end{proof}

\section{Approximation of the Forecasting Map} \label{sec:forca}

Having shown that map $\Psid_F$ exists, we now demonstrate that
we may approximate the map to arbitrary accuracy. We first 
a general universal approximation result
for operator approximation in Banach space;
we then apply this result to an attention-based 
neural operator designed to approximate the forecasting problem.
The following proposition is proved in 
\cite[Theorem~22]{calvello2024continuum}. 

\begin{proposition}
\label{thm:UANeuralOperator}
Let $D \subset \R^d$ be a bounded domain with Lipschitz boundary, and fix integers $s,s'\geq 0$. If $
\Psi^\dagger\colon C^{s}(\overline{D};\,\R^{r}) \to C^{s'}(\overline{D};\,\R^{r'})$
is a continuous operator and $K \subset C^{s}(\overline{D};\,\R^{r})$ a compact set, then for any $\varepsilon>0$ there exists a transformer neural operator 
$\Psi(\placeholder; \theta)\colon K \subset C^{s}(\overline{D};\,\R^{r}) \to C^{s'}(\overline{D};\,\R^{r'})$
so that
\begin{equation}
\sup_{u\in K} \left\| \Psi^\dagger(u) - \Psi(u;\theta) \right\|_{C^{s'}} \le \varepsilon .
\label{eq:1234}
\end{equation}
\end{proposition}

\begin{remark}[Choice of Compact Set]
The parameters $\theta$ in \Cref{thm:UANeuralOperator}, and the size of the architecture, depend on $\Psi^\dagger$, on $K$ and on $\varepsilon$. Although compact subsets of an infinite-dimensional Banach space have empty interior, they can carry almost all of the probability mass: for any Borel probability measure $\mu$ on a separable Banach space and any $\delta>0$ there is a compact set $K_\delta$ with $\mu(K_\delta)>1-\delta$; see, for example, \cite[Appendix C.1]{lanthaler2022error}. In our setting, the relevant compact set is $K = S^I_{[0,T],p}$, the set of projected trajectory segments with initial conditions in the compact set $I$. For dynamics possessing a compact attractor, the set of $p$-projections on this attractor is a natural choice.
\end{remark}

Leveraging the approximation properties of transformer
neural operators, encapsulated in the previous proposition,
we may prove: 

\begin{theorem}[Universal Approximation for Forecasting]
\label{thm:UAforecasting}
Consider the dynamical system \eqref{eq:dynamical_system_intro} satisfying
regularity \Cref{assump:regularity}, for integer $k$, and the observability \Cref{assump:observability} at some $(\mathfrak{p},\mathfrak{q})\in\R^{d_p + d_q}$, for integer $n \le k.$ For any $\epsilon > 0$ there exists a transformer neural operator $\Psi(\placeholder;\theta)\colon S^I_{[0,T],p} \subset C^k ([0,T];\R^{d_p}) \to S^I_{[T,T+\tau],p} \subset C^k ([T,T+\tau];\R^{d_p})$ satisfying
\[\sup_{p \in S^I_{[0,T],p}} \|\Psi_F^\dagger(p) - \Psi(p;\theta)\|_{C^k([T,T+\tau];\R^{d_p})} \leq \epsilon. \]
\end{theorem}
\begin{proof}
    We know that $\Psi_F^{\dagger}$ exists and is continuous since the assumptions of \Cref{prop:psiF} are satisfied. Recall that $S^I_{[0,T],p}$ is compact. We apply the result of \Cref{thm:UANeuralOperator} to an operator defined via linear homeomorphisms of $\Psi_F^\dagger$, which we will then use to establish the approximation result in its final form. Indeed, we define the translation map $\Lambda_T:C^k([T,T+\tau];\R^{d_p})\to C^k([0,\tau];\R^{d_p})$ so that
    \[(\Lambda_T \varphi)(s) = \varphi(T+s), \qquad s\in[0,\tau], \]
    for any $\varphi\in C^k([T,T+\tau];\R^{d_p})$. We note that 
    \[\|\varphi\|_{C^k([T,T+\tau];\R^{d_p})} = \|\Lambda_T\varphi\|_{C^k([0,\tau];\R^{d_p})}.\]
    It is also possible to define the rescaling operator $R_\tau:C^k([0,\tau];\R^{d_p})\to C^k([0,1];\R^{d_p})$ so that
    \[
    (R_\tau\varphi)(s) = \varphi(\tau s), \qquad s\in[0,1],
    \]
    for any $\varphi\in C^k([0,\tau];\R^{d_p})$. We similarly define the rescaling operator $R_T$. We note that these rescaling operators are linear homeomorphisms, preserving equivalence of norms up to multiplicative constants. We therefore define the translated and rescaled forecasting operator $\widetilde{\Psi}^\dagger_F:R_T(S^I_{[0,T],p})\subset C^k([0,1];\R^{d_p})\to C^k([0,1];\R^{d_p})$ defined via the composition
    \[
    \widetilde{\Psi}^\dagger_F = R_\tau \circ \Lambda_T \circ {\Psi}^\dagger_F \circ R_T^{-1},
    \]
    which is continuous by continuity of all the operators in the composition. We note that $R_T(S^I_{[0,T],p})$ is compact since $R_T$ is a homeomorphism and therefore preserves compactness. Hence, by \Cref{thm:UANeuralOperator}, it is possible to deduce that for any $\delta>0$ there exists a neural operator $\widetilde{\Psi}(\placeholder,\theta): R_T(S^I_{[0,T],p})\subset C^k([0,1];\R^{d_p})\to C^k([0,1];\R^{d_p})$ satisfying
    \begin{equation}
    \label{eq:approximant}
    \sup_{p \in R_T(S^I_{[0,T],p})} \|\widetilde{\Psi}_F^\dagger(p) - \widetilde{\Psi}(p;\theta)\|_{C^k([0,1];\R^{d_p})} \leq \delta. 
    \end{equation}
    This indeed follows from a direct application of \Cref{thm:UANeuralOperator} with $D = [0, 1]$, $r = d_p$, $r' = d_p$, $\Psi^{\dagger} = \widetilde{\Psi}_F^{\dagger}$, and $K = R_T(S^I_{[0,T],p})$. Defining $\Psi(\placeholder,\theta):C^k ([0,T];\R^{d_p}) \to C^k ([T,T+\tau];\R^{d_p})$ via the composition
    \[
    \Psi(\placeholder,\theta) = \Lambda_T^{-1}\circ R_\tau^{-1} \circ \widetilde{\Psi}(\placeholder;\theta)\circ R_T,
    \]
    it is readily deduced that for any $p \in S^I_{[0,T],p}$ it holds that
    \[
    \|\Psi_F^\dagger(p) - \Psi(p;\theta)\|_{C^k([T,T+\tau];\R^{d_p})} = C\|\widetilde{\Psi}_F^\dagger(R_Tp) - \widetilde{\Psi}(R_Tp;\theta)\|_{C^k([0,1];\R^{d_p})},
    \]
    for a constant $C$ depending on $T$ and $\tau$. Therefore, choosing $\epsilon=\delta/C$ and applying the result established in \eqref{eq:approximant} yields the conclusion.
    \end{proof}
We note that the proof of \Cref{thm:UAforecasting} uses only the approximation property of \Cref{thm:UANeuralOperator}; hence the result holds for any parametric family of operators satisfying a uniform approximation bound of the form \eqref{eq:1234}; other neural operators satisfying universal approximation guarantees
over compact sets may be used. The same consideration holds for \Cref{thm:UAsmoothing} in the next section.

\section{Extensions to the Smoothing Problem} \label{sec:smooth}

The following two theorems prove existence and approximability
of the smoothing map $\Psid_S.$ Proofs are very similar to those
for the forecasting map and may be found in Section \ref{sec:PST}.
Using \Cref{prop:W} we first establish the existence of an operator 
mapping observed $p$ trajectories to unobserved $q$ trajectories.

\begin{theorem}[Existence of map $\Psi_S^\dagger:p\mapsto q$]
\label{prop:psiS}
Consider the dynamical system \eqref{eq:dynamical_system_intro} satisfying
regularity \Cref{assump:regularity}, for integer $k$, and the observability \Cref{assump:observability} at some $(\mathfrak{p},\mathfrak{q})\in\R^{d_p + d_q}$, for integer $n \le k.$
Let $U \ni (\fp, \fq)$ and $V \ni LF^{(n)}(\fp, \fq)$ be the open sets given by the inverse function theorem. Then, for every compact $I \subset U$, there exists continuous operator $\Psi_S^\dagger \colon S^I_{[0,T],p} \subset C^k([0,T];\R^{d_p}) \to S^I_{[0,T],q} \subset C^k([0, T]; \R^{d_q})$ such that, for every $(p, q) \in S^I_{[0,T]}$
and $t \in [0,T]$,
\begin{equation}
    \label{eq:p_to_q_map1}
    \Psi_S^\dagger(p)(t) = q(t).
\end{equation}
\end{theorem}

Having established existence of $\Psi_S$ we now show that we may 
approximate it, to arbitrary accuracy, using an attention-based
neural operator.

\begin{theorem}[Universal Approximation for Smoothing]
\label{thm:UAsmoothing}
Consider the dynamical system \eqref{eq:dynamical_system_intro} satisfying 
regularity \Cref{assump:regularity}, for integer $k$, and the observability \Cref{assump:observability} at some $(\mathfrak{p},\mathfrak{q})\in\R^{d_p + d_q}$, for integer $n \le k.$
For any $\epsilon > 0$ there exists a transformer neural operator $\Psi(\placeholder;\theta): S^I_{[0,T],p} \subset C^k ([0,T];\R^{d_p}) \to  S^I_{[0,T],q} \subset C^k ([0,T];\R^{d_q})$ satisfying
\[\sup_{p \in S^I_{[0,T],p}} \|\Psi_S^\dagger(p) - \Psi(p;\theta)\|_{C^k} \leq \epsilon. \]
\end{theorem}

\section{Data-Driven Approximation of Smoothing and Forecasting Maps}
\label{sec:experiments}

The theoretical developments of the three preceding sections
prove the existence of, and universal approximation theorems for,
maps that underpin data-driven forecasting and smoothing. In this section we demonstrate that
it is possible to learn these maps, purely from data, in practice. In \Cref{ssec:esu} we describe the experimental
set-up in general, carefully distinguishing between the data assumed to be available for
forecasting and for smoothing; in particular we reiterate that for forecasting no data on the
$q$ variable is required to train the model. 
This is followed, in \Cref{ssec:summ}, by a summary of the numerical results.
We describe details of the experiment with the Lorenz `63 \cite{lorenz63} and Lorenz `96 \cite{lorenz96} dynamical systems in \Cref{subsec:lorenz63} and \Cref{subsec:lorenz96}, respectively; experiments with the Kuramoto-Sivashinsky equation \cite{kuramoto1978diffusion, michelson1977nonlinear} are contained in \Cref{subsec:ks}. The code repository pertaining to these experiments is publicly available at \url{https://github.com/EdoardoCalvello/operator-learning-forecasting}.

\subsection{Experimental Set-Up} \label{ssec:esu} 
 For the three dynamical systems considered we will approximate the operators
 \begin{equation}
\label{eq:ops_numerics}
\Psi^\dagger_{F}:\{p(t)\}_{t\in [0,T]}\mapsto \{p(t)\}_{t\in [T,T+\tau]} \quad \& \quad
    \Psi^\dagger_{S}:\{p(t)\}_{t\in [0,T]}\mapsto \{q(t)\}_{t\in [0,T]}.
\end{equation}
Only local existence of the operators in \eqref{eq:ops_numerics} is established in this paper: \Cref{prop:psiF}
and \Cref{prop:psiS}. However we assume global existence for our numerical investigations.
Establishing the existence of the global operators in \eqref{eq:ops_numerics} constitutes an important avenue for future work, which we highlight in \Cref{sec:conclusions}.

To train the neural operators solving the forecasting problem
we make the following assumption about available data:
\begin{dataassumption} \label{da:F}
 We are given $J$ input-output pairs in the form
 $$\{p^{(j)}|_{t \in [0,T]}, p^{(j)}|_{t \in [\tau,T+\tau]} \}_{j=1}^J.$$
\end{dataassumption}
Notice that no knowledge of $q$ is required, structurally or in terms of available data.
To train the neural operators solving the smoothing problem, however
requires knowledge of $q$;
we make the following assumption about available data:
\begin{dataassumption} \label{da:S}
 We are given $J$ input-output pairs in the form
 $$\{p^{(j)}|_{t \in [0,T]}, q^{(j)}|_{t \in [0,T]} \}_{j=1}^J.$$
\end{dataassumption}
In a practical application the data may be acquired through physical experiments, or
through simulation of the dynamical system \eqref{eq:dynamical_system_intro} for $(p,q).$
Here, as proof of concept, we use the latter approach. We generate i.i.d. data pairs, although have i.i.d. pairs is not a requirement of the computational methodology.
We simulate the dynamical system, given i.i.d. initial conditions $\bigl(p^{(j)}(0),q^{(j)}(0)\bigr)\sim\nu$. Measure $\nu$ is computed as the pushforward, over some burn-in time, of the simpler measure $\nu_0.$ Burn-in is used to ensure that $\nu$ is close to the invariant measure
supported on the global attractor of the dynamical system.

We train transformer neural operators \cite{calvello2024continuum} to construct approximations of $\Psi^\dagger_{F}$ and $\Psi^\dagger_{S}$ picked from the parametric classes
$$\Psi_{F}:C([0,T])\times \Theta \to C([T,T+\tau]) \quad \& \quad \Psi_{S}:C([0,T])\times \Theta \to C([0,T])$$
respectively. In all cases we use the transformer neural operator based on self-attention from \cite{calvello2024continuum} to approximate the $\Psi^\dagger_{S}$. To approximate $\Psi^\dagger_{F}$ in the context of Lorenz `63, Lorenz `96 and Kuramoto-Sivashinsky models we employ a variant of the transformer neural operator architecture which uses cross-attention. Architectural details as well as the universal approximation properties for these neural operators are discussed in more depth in \Cref{appendix:architectures}. We are not claiming
that this is the optimal parametric family to deploy for the forecasting and smoothing problems But we do note
that: (i) the use of neural operator allows transfer of parameters, once learned, between different discretizations; (ii) the use of attention is natural as it is well-adapted to the identification of correlation structure in time.

\subsection{Summary of Results from Numerical Experiments} \label{ssec:summ}

We provide here a summary of the experimental results. Details of the exact experimental set-up, for each problem, are
considered in the subsections that follow. We start, in \Cref{tab:smoothing_results}, where
we collect the results of the smoothing experiments for the various systems, reporting relative $L^2$ errors between the true unobserved trajectory and the prediction generated by the neural operator. All errors are reported on test data not used for training. 
 
\begin{table}[h]
\footnotesize
\centering
\caption{Smoothing performance on Lorenz `63 (L63), Lorenz `96 (L96), and the Kuramoto--Sivashinsky equation (KSE).}
\label{tab:smoothing_results}
\begin{tabular}{lccccccc}
\toprule
\textbf{System} 
& \textbf{Avg. Rel. $L_2$} 
& \textbf{Med. Rel. $L_2$} 
& \textbf{Std. Rel. $L_2$} 
& \textbf{Min. Rel. $L_2$} 
& \textbf{Max. Rel. $L_2$} 
\\
\midrule
L63 
& 0.001528 
& 0.001432 
& 0.000450 
& 0.000883 
& 0.008900 
\\
L96 
& 0.024080 
& 0.022399 
& 0.007435 
& 0.009882 
& 0.081556 
\\
KSE 
& 0.007607 
& 0.007203 
& 0.001598 
& 0.005200 
& 0.015436 
\\
\bottomrule
\end{tabular}
\end{table}

It is of interest to understand the robustness of the operator learning methodology under training and inferencing with noisy input observations. The noisy setting is of practical importance in data assimilation, where observations are often assumed to be corrupted by noise. We perform an experiment to test this robustness in the context of smoothing for Lorenz `63. Indeed, we train a model to reconstruct the $y,z$ trajectories from $x$ trajectories corrupted by noise, both in training and inference. We report the median relative $L^2$ errors achieved on corresponding test sets in \Cref{tab:l63_obs_noise}. For observational noise levels up to $\sigma_{\mathrm{obs}} = 0.01$, the median relative error remains essentially unchanged from the noise-free case. Even at $\sigma_{\mathrm{obs}} = 1$, it remains small at $0.0124$, suggesting robustness to noise.

% Requires \usepackage{booktabs}
\begin{table}[t]
  \centering
  \caption{Lorenz `63 smoothing experiment with noisy observations. The model observes $x$
  corrupted by additive Gaussian noise of standard deviation $\sigma_{\mathrm{obs}}$
  and reconstructs $y$ and $z$. The same noise level is used for training and
  inference, and errors are computed against the noise-free trajectories. We report
  the median relative $L^2$ error over 1000 held-out trajectories.}
  \label{tab:l63_obs_noise}
  \begin{tabular}{lccccc}
    \toprule
    $\sigma_{\mathrm{obs}}$ & $0$ & $0.001$ & $0.01$ & $0.1$ & $1$ \\
    \midrule
    Median rel.\ $L^2$ error & $0.00143$ & $0.00156$ & $0.00172$ & $0.00401$ & $0.0124$ \\
    \bottomrule
  \end{tabular}
\end{table}

In \Cref{tab:forecasting_results} we report the results of the forecasting experiments. For each system, we report the mean, median and standard deviation of the relative $L^2$ errors obtained on the sample trajectories in the test set. We also report the relative improvement achieved by the neural operator prediction when compared to
a polynomial extrapolation of the
observed history. To obtain this baseline, for each trajectory and each state component independently,
we fit a polynomial $\sum_{j=0}^{r} a_j t^j$ of degree $r$ by least squares
to the last $m \ge r+1$ observed values $x(t_{N-m+1}),\dots,x(t_N)$, and evaluate
it at the forecast times. We note that the case $r=0$, $m=1$ recovers the constant baseline
$\hat{p}(s) = p(T)$, for $s>T$. For each system and prediction horizon, $(r,m)$ is selected
from $r\in\{0,\dots,6\}$ and $m\in\{1,2, \dots,8,10, 12, 15, 20\}$ by minimizing the median relative
$L^2$ error on the validation set, and the selected baseline is then evaluated on
the test set.
We note that the chaotic nature of the problems considered means that long-term trajectory forecasting is difficult. However we can compose forecasts and look
at the resulting statistics of the time-series -- in particular the invariant measure of the resulting stationary process; these are well-approximated. Thus, for each of the systems, we report in the respective subsections the statistics of the long-time forecasts obtained via composition of the trained neural operators. The results in \Cref{tab:forecasting_results} along with the accurate prediction of trajectory statistics demonstrate the success of using a neural operator forecasting approach.

\begin{table}[htbp]
\footnotesize
\centering
\caption{Forecasting performance on Lorenz `63 (L63), Lorenz `96 (L96), and the Kuramoto--Sivashinsky equation (KSE). In addition to the relative $L^2$ errors achieved on the test set, we also report the relative improvement of median $L^2$ errors achieved with the transformer neural operator compared to the best polynomial extrapolation.}
\label{tab:forecasting_results}
\begin{tabular}{lcccc}
\toprule
\textbf{System} 
& \textbf{Avg. Rel. $L_2$} 
& \textbf{Med. Rel. $L_2$} 
& \textbf{Std. Rel. $L_2$} 
& \textbf{Rel. Improvement (\%)} \\
\midrule
L63  & 0.01267 & 0.00538 & 0.04392 & 99.5 \\
L96  & 0.00524 & 0.00490 & 0.00180 & 98.8 \\
KSE  & 0.0304 & 0.02872 & 0.01134 & 94.0 \\
\bottomrule
\end{tabular}
\end{table}

In the following subsections, we describe the experimental set-up for the various systems and provide analysis of the numerical results.

\subsection{Lorenz `63}
\label{subsec:lorenz63}

We study the Lorenz `63 dynamical system \cite{lorenz63} in the form given in \eqref{eq:L63_dynamic}.
For the purposes of our experiments, the parameters are set to the classical values
$\sigma = 10,\ \rho = 28,\ \beta = \tfrac{8}{3}$ at which the system has provable stable chaotic and statistical properties \cite{tucker1999lorenz,holland2007central}. The initial conditions at time $t=-20$, for both training and test sets, are sampled i.i.d from the uniform distributions $U([-15,15])$, $U([-15,15])$ and $U([0,40])$, respectively; this defines $\nu_0.$ We use a burn-in time of $20$, ensuring that the trajectories are close to the global attractor at time $t=0$, sampled from measure $\nu$, approximating the invariant measure supported on the global attractor. For the \Cref{prob:smoothing} smoothing experiment we use the $x$ trajectory on a time interval of $[0,5]$ as input data and predict the coupled $y,z$ trajectories over the same time interval. For practical implementation, we use points on the trajectories sampled at uniform intervals $\Delta t= 0.02$. For the \Cref{prob:forecasting} forecasting experiment we use the $x$ trajectory on a time interval of $[0,2]$ as input data and predict the coupled $x$ trajectory over the future time $[2,4]$. For practical implementation, we use points on the input and output trajectories sampled at uniform intervals $\Delta t= 0.01$. We train the neural operator solving the smoothing problem using $10000$ training trajectories and the one solving the forecasting problem using $200000$ trajectories. For both experiments we use a test set of $2000$ trajectories. We note the higher data complexity required to solve the forecasting problem.

We summarize the parameter and experimental settings for the Lorenz `63 dynamical system in \Cref{tab:lorenz63_settings}.
\begin{table}[h]
\footnotesize
\centering
\caption{Lorenz `63 system settings}
\label{tab:lorenz63_settings}
\begin{tabular}{lcc}
    \toprule
    \textbf{Category} & \textbf{Smoothing} & \textbf{Forecasting} \\
    \midrule
    Parameters 
        &\multicolumn{2}{c}{ $\sigma = 10,\ \rho = 28,\ \beta = \tfrac{8}{3}$} \\
    \midrule
    Observed trajectory
        & $x \in C\bigl([0,5]; \bbR \bigr)$ 
        & $x \in C\bigl([0,2]; \bbR \bigr)$ \\
    \midrule
    Unobserved trajectory 
        & $(y,z) \in C\bigl([0,5]; \bbR^2 \bigr)$ 
        & $(y,z) \in C\bigl([2,4]; \bbR^2 \bigr)$ \\
    \midrule
    Observation time step 
        & $\Delta t = 0.02$
        & $\Delta t = 0.01$ \\
    \bottomrule
\end{tabular}
\end{table}
In \Cref{fig:L63_smoothing} we demonstrate the qualitative results of the smoothing experiment. Indeed, we display the predicted trajectories corresponding to the median and highest relative $L^2$ errors in the test set compared to true trajectories. The panels showcase predictions that are qualitatively indistinguishable from the truth, indicating the success of the purely data-driven smoothing approach.
\begin{figure}
    \centering
    \includegraphics[width=\linewidth]{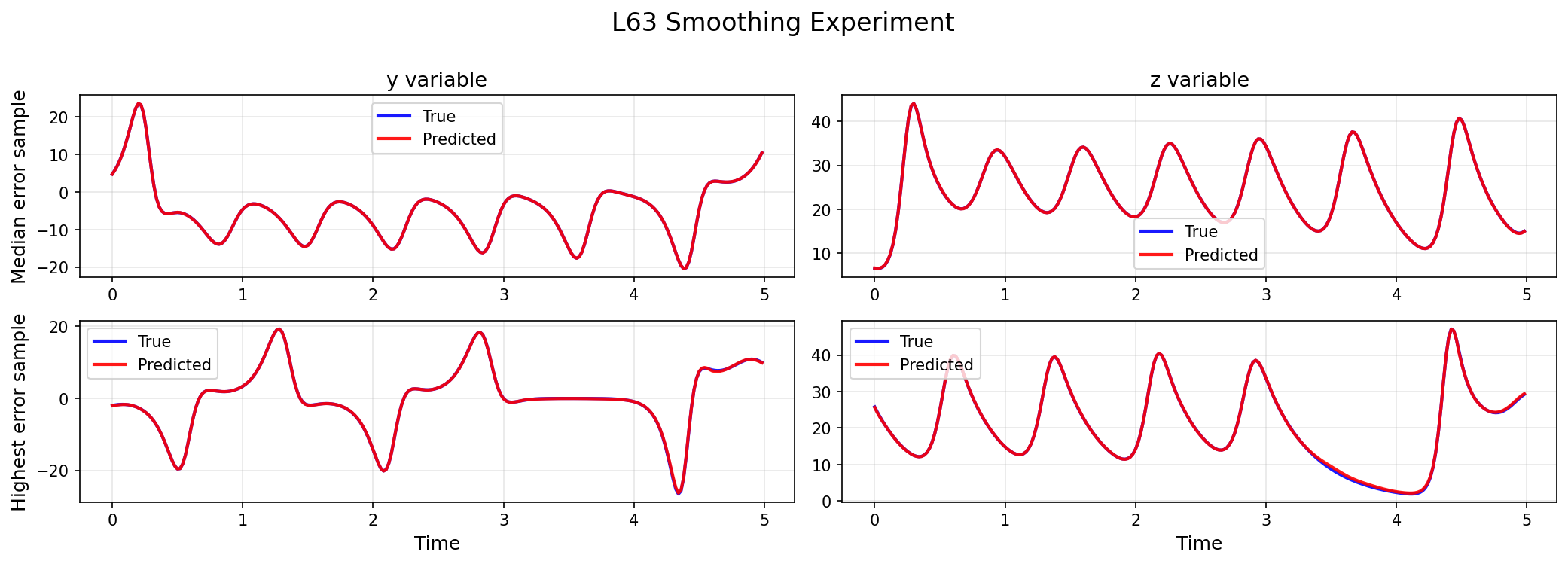}
    \caption{Lorenz `63. Median and worst-case relative $L^2$ error samples for smoothing experiment involving prediction of $(y,z)$ from $x$.}
    \label{fig:L63_smoothing}
\end{figure}

To highlight the importance of observability, we experiment with the smoothing setting of mapping the $z$ trajectory over the time interval $[0,5]$ to the $x,y$ trajectories over the same time interval. We use the same model and training setup as before. \Cref{fig:L63_smoothing_xy} displays the predicted trajectories corresponding to the median and highest relative $L^2$ errors in the test set compared to true trajectories. The panels showcase the failure in the predictions, which achieve a median relative $L^2$ error of $1$; it is also interesting to note that the model seems to predict the $0$ trajectory, possibly due to the reflection argument discussed in \Cref{ex:Lorenz}.
\begin{figure}
    \centering
    \includegraphics[width=\linewidth]{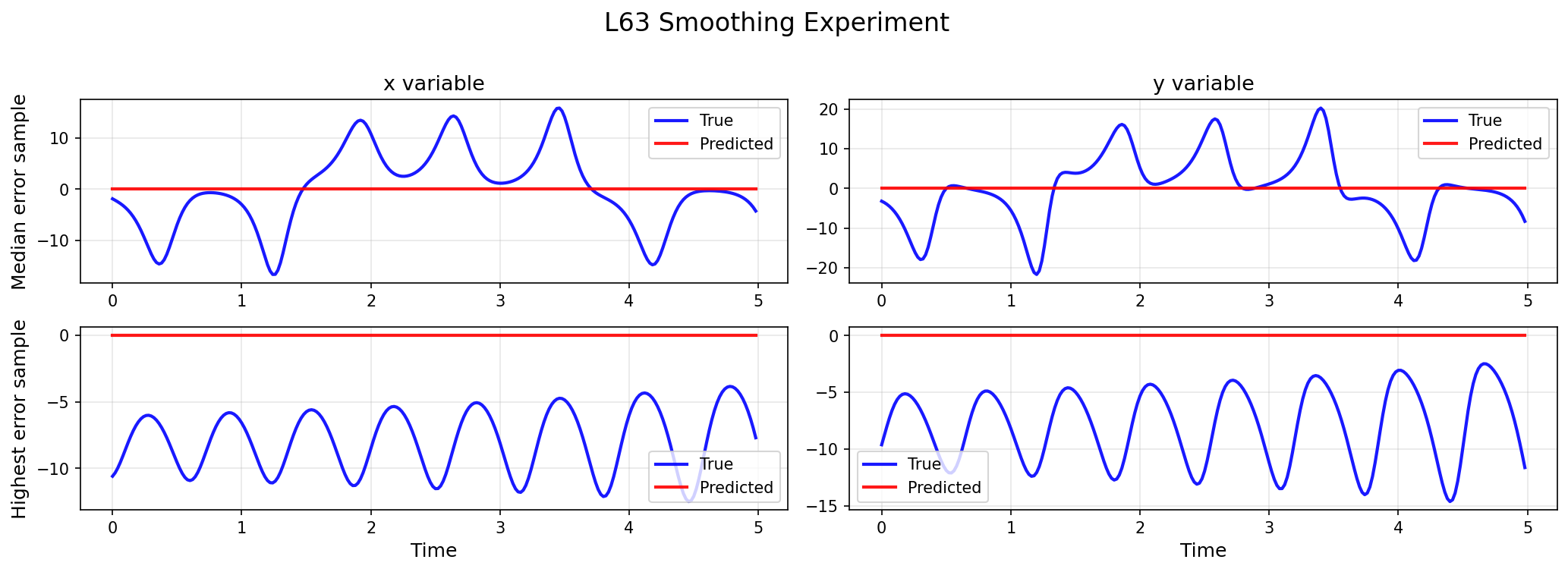}
    \caption{Lorenz `63. Median and worst-case relative $L^2$ error samples for smoothing experiment involving prediction of $(x,y)$ from $z$.}
    \label{fig:L63_smoothing_xy}
\end{figure}

In \Cref{fig:L63_forecasting_all} we demonstrate the qualitative results of the forecasting experiment, based on data from the $x$ coordinate. Indeed, in the top panel of \Cref{fig:L63_forecasting_all} we display the predictions for the trajectories corresponding to the median and highest relative $L^2$ errors in the test set compared to ground truths. We note that the prediction for the test sample presenting the largest error still resembles a possible trajectory of the Lorenz `63 system. Indeed, the prediction diverges to another lobe of the attractor; this behavior is to be expected in prediction problems in the context of chaotic systems. With the goal of obtaining forecasts on longer time horizons, we experiment with composition of the learned forecasting map $\Psi_F:C([0,2])\to C([2,4])$ to obtain $\Psi_F^n:C([0,2])\to C([2,2+2n])$. In the bottom panel of \Cref{fig:L63_forecasting_all}, we display the continuation of the one-step forecasts in the top panel, by $n=29$ compositions of the learned model. Clearly, while the predicted trajectories diverge from the ground truth, the predictions continue to resemble plausible trajectories of the Lorenz `63 system. In \Cref{fig:L63_forecasting_hist}, for $n=500$ we compare the distribution (histogram) of points in the predicted future trajectories $\widehat{p}$ to points in the ground truth trajectories $p$. The matching of distributions, and hence statistics, indicates the success in predicting valid trajectories from the attractor of the Lorenz system, despite the chaotic nature of the equations and hence per-trajectory divergence. As further validation of the quantitative success of the purely data-driven forecasting approach, we recall the numerical comparison with the prediction given by the constant value of $x(T)$ for $T=2$, that is taking the predicted trajectory to be $\widehat{p}(t)=x(2)$ for $t\in[2,4]$. \Cref{tab:forecasting_results} shows the improvement in accuracy yielded by the cross-attention based transformer neural operator, in comparison to this naive approach.

% Requires \usepackage{subcaption}
\begin{figure}[t]
  \centering
  \begin{subfigure}{\linewidth}
    \centering
    \includegraphics[width=0.5\linewidth]{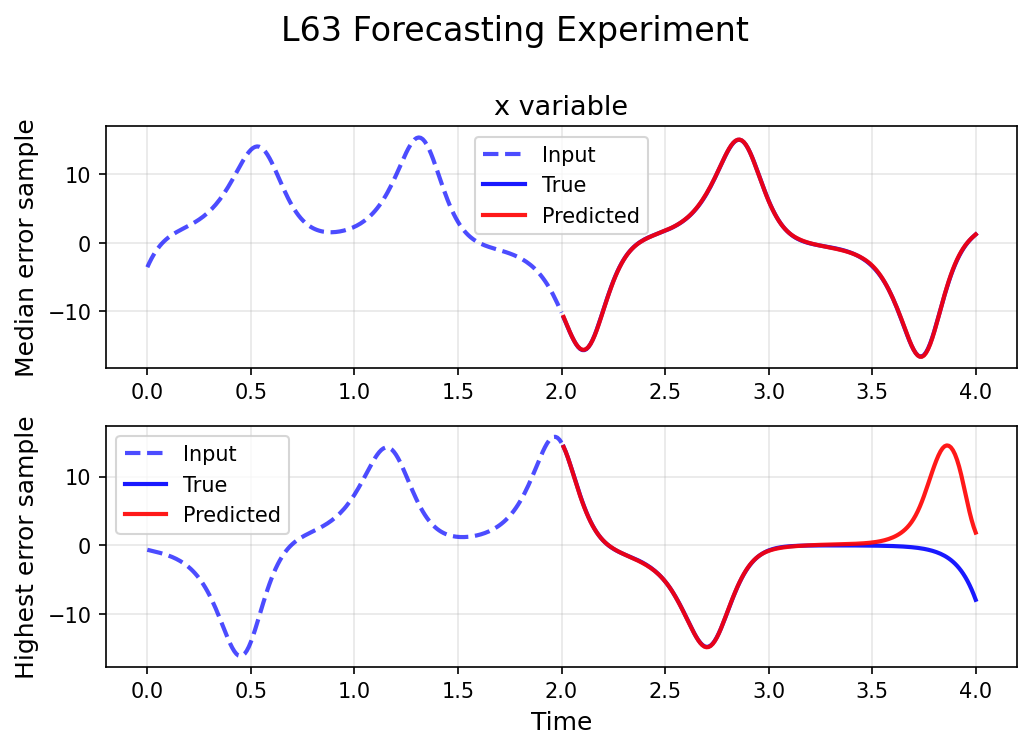}
  \end{subfigure}

  \medskip
  \begin{subfigure}{\linewidth}
    \centering
    \includegraphics[width=\linewidth]{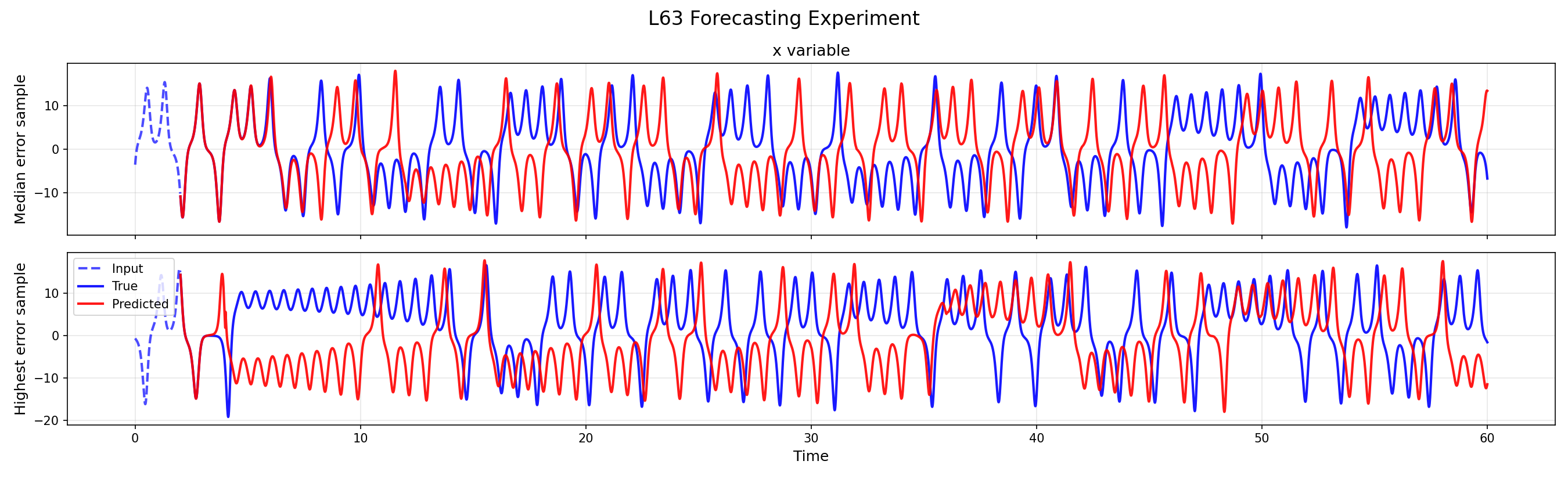}
  \end{subfigure}
    \caption{Lorenz `63. Top: median and worst-case relative $L^2$ error samples in forecasting. Bottom: forecast obtained by composing the learned map on one-step median and worst-case relative $L^2$ error samples.}
    \label{fig:L63_forecasting_all}
\end{figure}

\begin{figure}[htbp]
    \centering
    \includegraphics[width=0.5\linewidth]{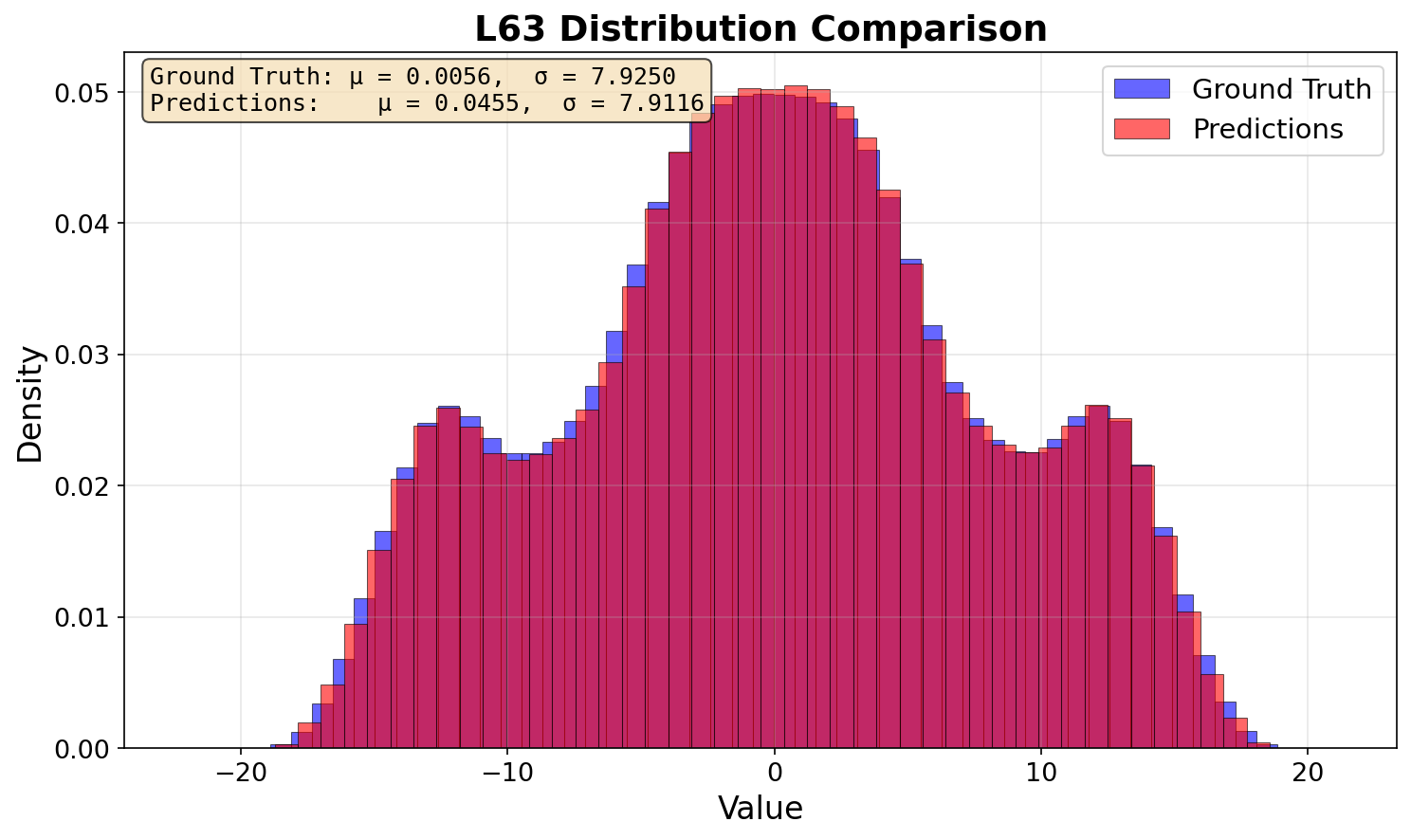}
    \caption{Lorenz `63. Distribution of trajectory predictions under composition of the learned forecasting map; the dark shade of red represents overlap between the blue ground truth distribution and the light red predicted distribution. The prevalence of the dark shade of red thus indicates matching of distributions.}
    \label{fig:L63_forecasting_hist}
\end{figure}

\subsection{Lorenz `96}
\label{subsec:lorenz96}

The Lorenz `96 model \cite{lorenz96} consists of a linearly damped and externally forced system equipped with an energy-conserving quadratic nonlinearity, which generates cyclic interactions among the state variables. These properties make the model a standard testbed for investigating atmospheric predictability. The system dynamics are governed by the following equations:
\begin{equation}\label{eq:L96_equations}
    \frac{du_i}{dt} = \bigr(u_{i+1} - u_{i-2}\bigl)u_{i-1} - u_i + F,\quad i=1,2,\ldots, d.
\end{equation}
Periodicity is imposed to define variables $u_j$ with indices outside the set $j=1,2,\ldots, d.$
For the purposes of our experiments we set $d=40$ and $F=8$. For all components $u_i$ the initial conditions for training and test sets are found by first selecting initial conditions drawn from $F+U([-1,1])$, 
i.i.d. in both index $i$ and random choice of initialization at time $t=-200$, and  then simulating the system using a burn-in time of $200$ to ensure the initializations of the trajectories at $t=0$ are sampled i.i.d. from a measure approximating the invariant measure supported on the global attractor at time $t=0$. For the smoothing experiment we take trajectories over a time interval of $[0,5]$ as input data and predict the unobserved trajectories over the same time interval; we implement using points on the trajectories sampled at uniform intervals $\Delta t= 0.02$. For the forecasting experiment we also use trajectories on a time interval of $[0,5]$ as input data and predict observed trajectories over the future time $[5,5.2]$. We implement using points on the input and output trajectories sampled at uniform intervals $\Delta t= 0.02$. We train the neural operator solving the smoothing problem using $10000$ training trajectories and the one solving the forecasting problem using $160000$ trajectories. For both experiments we use a test set of $2000$ trajectories. We again note the higher data complexity required to solve the forecasting problem. 

We summarize the parameter and experimental settings for the Lorenz `96 dynamical system in \Cref{tab:lorenz96_settings}.
\begin{table}[h]
\footnotesize
\centering
\caption{Lorenz `96 system settings. In the smoothing setting, we observe all of the first $20$ components and subsequently only components with odd index. In the forecasting setting, we observe three every four components.}
\label{tab:lorenz96_settings}
\begin{tabular}{lcc}
    \toprule
    \textbf{Category} & \textbf{Smoothing} & \textbf{Forecasting} \\
    \midrule
    Parameters
       &\multicolumn{2}{c}{$F=8$} \\
    \midrule
    Observed
        & $(u_1,\ldots,u_{21},u_{23},u_{25},\ldots,u_{39}) \in C([0,5];\mathbb{R}^{30})$ & $(u_1,u_{2},u_{3},u_{5},\ldots,u_{39}) \in C([0,5];\mathbb{R}^{30})$
          \\
    \midrule
    Unobserved 
        & $(u_{22},u_{24},\ldots,u_{40}) \in C([0,5];\mathbb{R}^{10})$
        & $(u_1,u_{2},u_{3},u_{5},\ldots,u_{39}) \in C([5,5.2];\mathbb{R}^{30})$ \\
    \midrule
    Obs. time
        &\multicolumn{2}{c}{ $\Delta t = 0.02$ }
        \\
    \bottomrule
\end{tabular}
\end{table}
In \Cref{fig:L96_smoothing} we demonstrate the qualitative results of the smoothing experiment. Indeed, we display the predicted trajectories corresponding to the median and highest relative $L^2$ errors in the test set compared to true trajectories. The median error panel demonstrates near-perfect overlap between prediction and truth, while the highest error panel exhibits overlap only after $t=1$, indicating a possible effect of synchronization \cite{pecora1990synchronization}. Together, these results demonstrate the qualitative success of the purely data-driven smoothing approach. 
\begin{figure}[htbp]
    \centering
    \includegraphics[width=\linewidth]{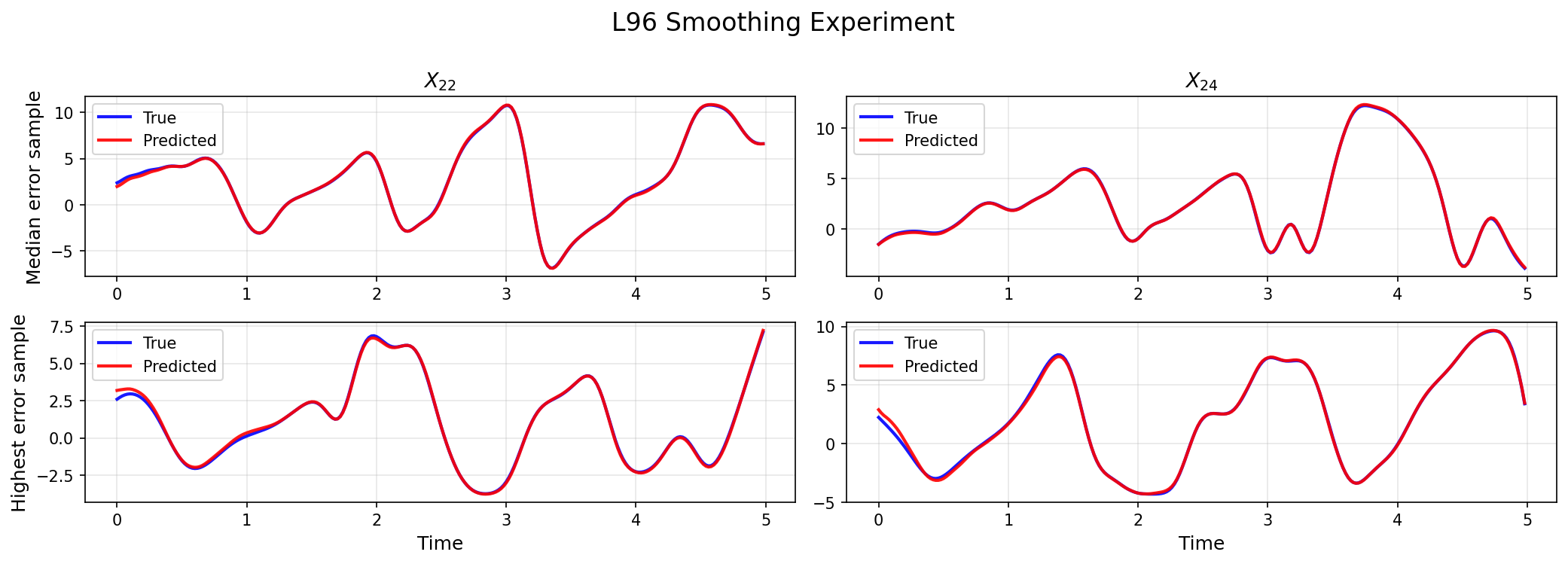}
    \caption{Lorenz `96. Median and worst-case performance on the Lorenz `96 system for smoothing.}
    \label{fig:L96_smoothing}
\end{figure}
In \Cref{fig:L96_forecast_spatiotemporal} we display results for the forecasting experiment. Here we show the pointwise errors relative to the root mean square of the ground truth (computed in space) for the sample in the test set achieving the median relative $L^2$ error. With the goal of obtaining forecasts on longer time horizons, we experiment with composition of the learned forecasting map $\Psi_F:C([0,5])\to C([5,5.2])$ to obtain $\Psi_F^n:C([0,5])\to C([5,5+0.2n])$. In \Cref{fig:L96_forecasting}, we show an example of predicted trajectories obtained by composing the map with $n=50$; the plot shows trajectory divergence after a handful of compositions, due to chaos and error compounding. However, in \Cref{fig:L96_forecasting_hist}, for $n=500$ we compare the distribution of points in the predicted future trajectories $\widehat{p}$ to the distribution of
points in the ground truth trajectories $p$. The matching of distributions, and hence statistics, indicates the success in predicting the invariant measure supported on the attractor of the Lorenz `96 system, despite the chaotic nature of the equations that results in trajectory divergence. As further validation of the quantitative success of the purely data-driven forecasting approach, we recall the numerical comparison with the prediction given by the constant value of $x(T)$ for $T=5$, that is taking the predicted trajectory to be $\widehat{p}(t)=p(5)$ for $t\in[5,5.2]$. \Cref{tab:forecasting_results} shows the improvement in accuracy yielded by the cross-attention based transformer neural operator, in comparison to this naive approach.
\begin{figure}[htbp]
    \centering
    \includegraphics[width=0.6\linewidth]{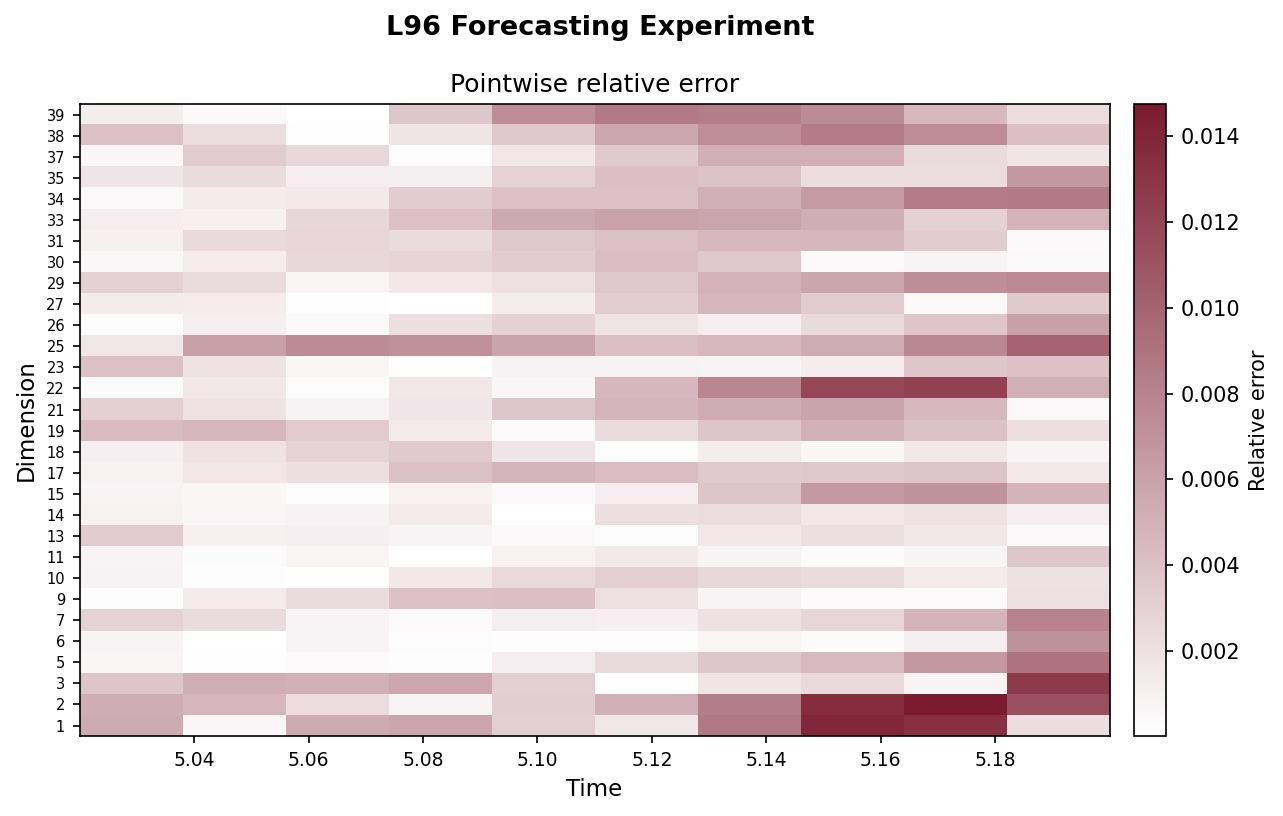}
    \caption{Lorenz `96. Pointwise errors relative to the root mean square of the ground truth (computed in space) for the sample in the test set achieving the median relative $L^2$ error.}
    \label{fig:L96_forecast_spatiotemporal}
\end{figure}
\begin{figure}[htbp]
    \centering
    \begin{subfigure}{0.48\linewidth}
        \centering
        \includegraphics[width=\linewidth]{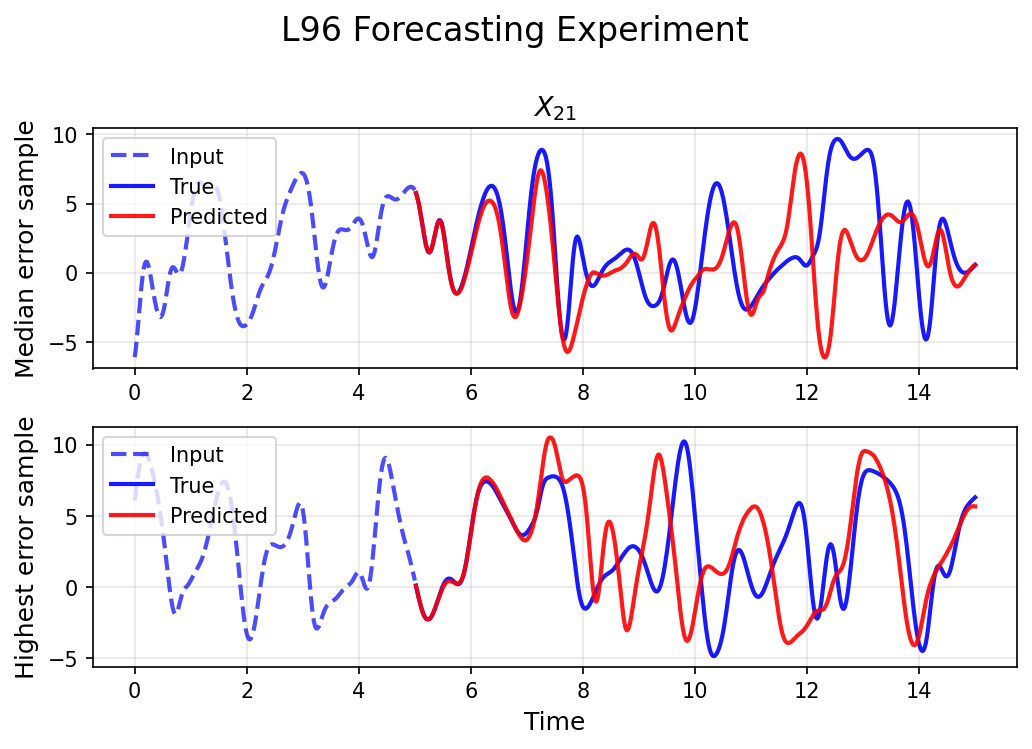}
        \caption{}
        \label{fig:L96_forecasting}
    \end{subfigure}\hfill
    \begin{subfigure}{0.48\linewidth}
        \centering
        \includegraphics[width=\linewidth]{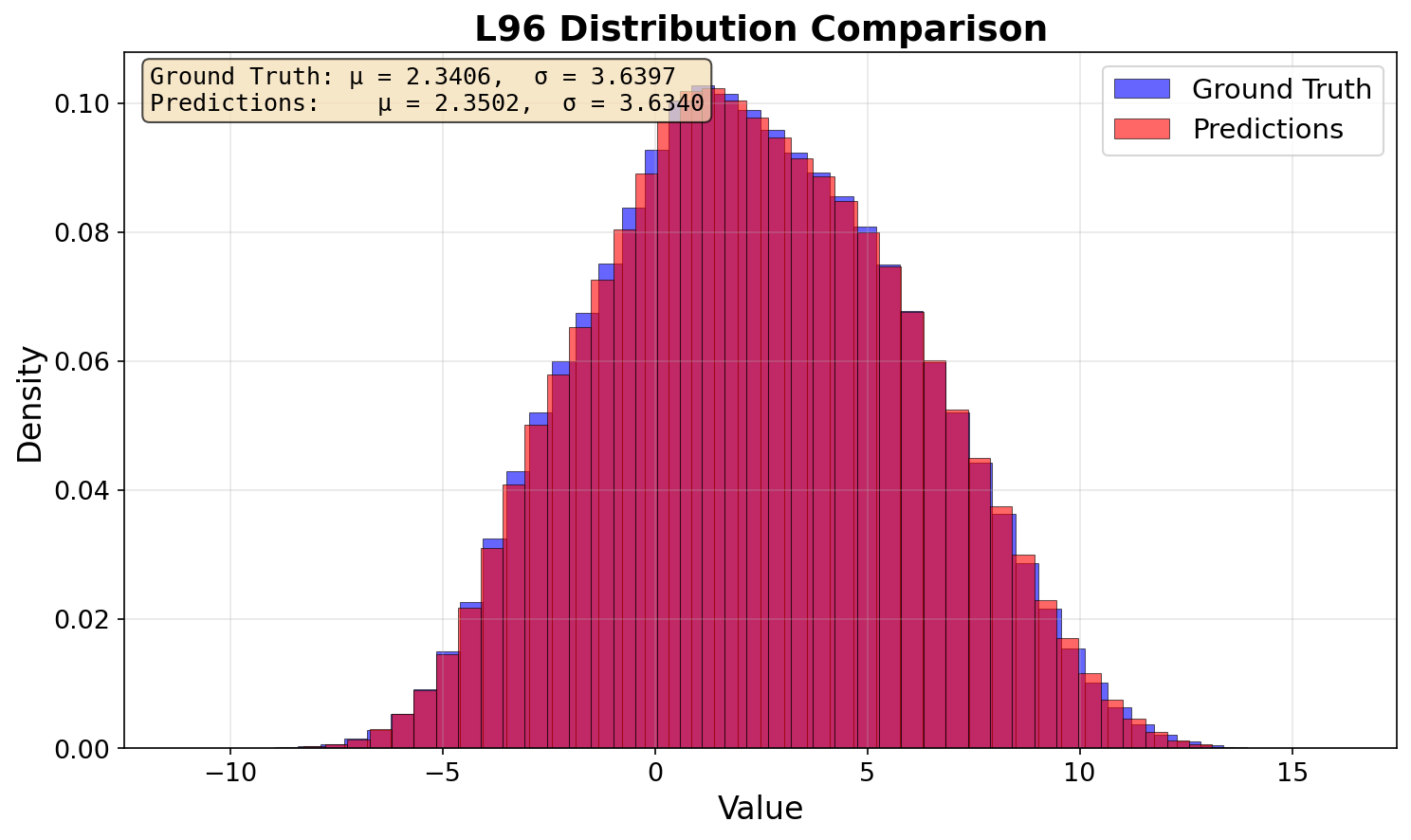}
        \caption{}
        \label{fig:L96_forecasting_hist}
    \end{subfigure}
    \caption{Lorenz `96. Left: forecast obtained by composing the learned map on one-step median and worst-case relative $L^2$ error samples. Right: distribution of trajectory predictions under composition of the learned forecasting map; the dark shade of red represents overlap between the blue ground truth distribution and the light red predicted distribution. The prevalence of the dark shade of red thus indicates matching of distributions.}
    \label{fig:L96_comparison}
\end{figure}

\subsection{Kuramoto-Sivashinsky}
\label{subsec:ks}

The Kuramoto–Sivashinsky (KS) equation \cite{kuramoto1978diffusion} is a well-studied nonlinear partial differential equation that models the development of instabilities in spatially extended systems. It exhibits rich spatiotemporal chaotic behavior and arises in a variety of physical contexts, including flame front propagation, thin film flows, and reaction–diffusion phenomena. With the periodicity in space imposed to identify the solution at $x=L$ and $x=0$, the one-dimensional KS equation for $u: [0,L] \times \mathbb{R}^+ \to \mathbb{R}$ takes the form
\begin{subequations}\label{eq:KS_equation}
\begin{align}
    \frac{\partial u}{\partial t} + \frac{\partial^4 u}{\partial x^4} + \frac{\partial^2 u}{\partial x^2} + u\frac{\partial u}{\partial x} &= 0, \quad
    (x,t) \in (0,L) \times \bbR^+, \\[1ex]
    u(x,0) & = u_0, \quad \forall x \in [0,L].
    \end{align}
\end{subequations}
For the purposes of our experiments, the domain size is set to $L=32\pi$. Numerically, the KS equation \eqref{eq:KS_equation} is discretized in space using a Fourier pseudospectral method, which naturally enforces periodic boundary conditions. Time integration is carried out using the exponential time-differencing fourth-order Runge–Kutta (ETDRK4) scheme \cite{kassam2005fourth}. This method treats the stiff linear operator analytically, thereby avoiding the explicit linear Courant–Friedrichs–Lewy (CFL) constraint associated with fully explicit Runge–Kutta schemes. As a result, the time step is selected according to accuracy requirements and nonlinear resolution rather than linear stability limitations. We simulate the dynamics using a burn-in time of $200$ to ensure the trajectories are independent and sampled from the attractor. The observed trajectories are taken to be filtered solutions of \eqref{eq:KS_equation}, obtained by zeroing out the high Fourier modes of $u$. We denote this filtered solution by $\widetilde{u}$. In the smoothing setting we experiment with, the goal is to recover the full solution 
$u$ from $\widetilde{u}$, where $\widetilde{u}$ is obtained by retaining the first $64$ modes of $u$ and zeroing out the rest. In the forecasting setting, $\widetilde{u}$ is obtained by retaining the first $32$ modes of $u$ and zeroing out the rest.

For the smoothing experiment we use the observed trajectories on a time interval of $[0,100]$ as input data and predict the unobserved trajectories over the same time interval. We implement using points on the trajectories sampled at uniform intervals $\Delta t= 0.25$. For the forecasting experiment we use the observed trajectories on a time interval of $[0,100]$ as input data and predict the the trajectories over the future time $[100,105]$.
 We implement using points on the input and output trajectories sampled at uniform intervals $\Delta t= 0.25$. 
We train the neural operator solving the smoothing problem using $10000$ training trajectories and the one solving the forecasting problem using $80000$ trajectories. For both experiments we use a test set of $2000$ trajectories. We note the higher data complexity required to solve the forecasting problem. We summarize the parameter and experimental settings for the Kuramoto-Sivashinsky equation in \Cref{tab:ks_settings}.
\begin{table}[h]
\footnotesize
\centering
\caption{Kuramoto--Sivashinsky (KS) system settings. Here, $u$ denotes the solution to the equation, while in the smoothing setting $\widetilde{u}$ is obtained by retaining the first $64$ modes of $u$ and zeroing out the rest. In the forecasting setting, $\widetilde{u}$ is obtained by retaining the first $32$ modes of $u$ and zeroing out the rest. }
\label{tab:ks_settings}
\begin{tabular}{lcc}
    \toprule
    \textbf{Category} & \textbf{Smoothing} & \textbf{Forecasting} \\
    \midrule
    Parameters
        &\multicolumn{2}{c}{ $L = 32\pi$}
        \\
    \midrule
     Spatial Grid
        &\multicolumn{2}{c}{
          $x_j = jL/128,\quad j=1,\ldots,128$ %\quad
          %$ \bigl(u(x_1),\ldots,u(x_{128})\bigr) \in C([0,100];\bbR^{128})$
          }
          \\
    \midrule
    Observed
        & $ \bigl(\widetilde{u}(x_1),\ldots,\widetilde{u}(x_{128})\bigr) \!\in C([0,100];\bbR^{128})$
        & $ \bigl(\widetilde{u}(x_1),\ldots,\widetilde{u}(x_{128})\bigr)  \!\in C([0,100];\bbR^{128})$ \\
    \midrule
    Unobserved
        & $ \bigl(u(x_1),\ldots,u(x_{128})\bigr) \!\in C([0,100];\bbR^{128})$
        & $ \bigl(\widetilde{u}(x_1),\ldots,\widetilde{u}(x_{128})\bigr)  \!\in C([100,105];\bbR^{128})$ \\
    \midrule
    Obs. time
        &\multicolumn{2}{c}{ $\Delta t = 0.25$ }
        \\
    \bottomrule
\end{tabular}
\end{table}
In \Cref{fig:kse_smoothing} we demonstrate the qualitative results of the smoothing experiment. Indeed, we display the predicted trajectories corresponding to the median and highest relative $L^2$ errors in the test set compared to true trajectories. The panels showcase predictions that are qualitatively indistinguishable from the truth, indicating the success of the purely data-driven smoothing approach.
\begin{figure}[htbp]
    \centering
    \includegraphics[width=\linewidth]{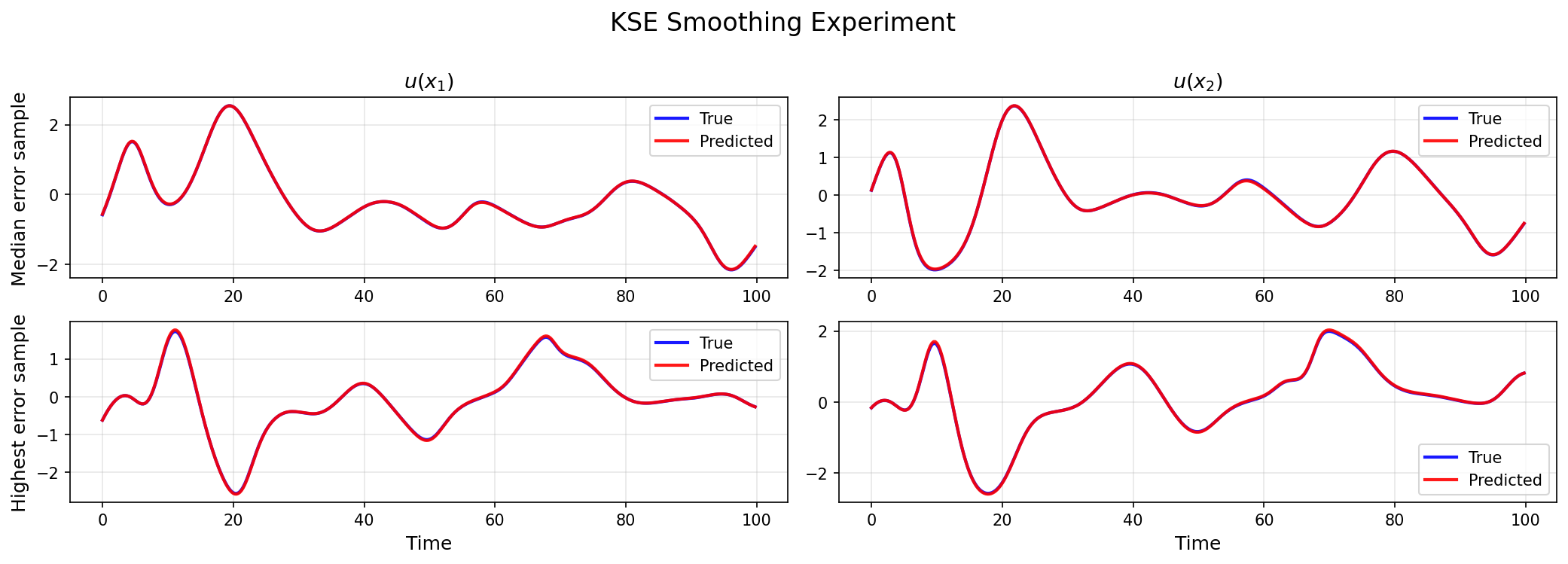}
    \caption{Median and worst-case performance on the KS system for smoothing.}
    \label{fig:kse_smoothing}
\end{figure}
In \Cref{fig:kse_spatiotemporal} we display results for the forecasting experiment. Here we show the pointwise errors relative to the root mean square of the ground truth (computed in space) for the sample in the test set achieving the median relative $L^2$ error. With the goal of obtaining forecasts on longer time horizons, we experiment with composition of the learned forecasting map $\Psi_F:C([0,100])\to C([100,105])$ to obtain $\Psi_F^n:C([0,100])\to C([100,100+5n])$. In \Cref{fig:kse_forecasting}, we show an example of predicted trajectories obtained by composing the map with $n=100$; the plot shows trajectory divergence after a handful of compositions, due to chaos and error compounding. However, in \Cref{fig:kse_forecasting_hist}, for $n=1000$ we compare the distribution of points in the predicted future trajectories $\widehat{p}$ to
the distribution of points in the ground truth trajectories $p$. The matching of distributions, and hence statistics, indicates the success in predicting valid trajectories from the attractor of the KS system, despite the chaotic nature of the equations and hence per-trajectory divergence. As further validation of the qualitative success of the purely data-driven forecasting approach, we recall the numerical comparison with the prediction given by the constant value of $p(T)$ for $T=100$, that is taking the predicted trajectory to be $\widehat{p}(t)=p(100)$ for $t\in[100,105]$. \Cref{tab:forecasting_results} shows the improvement in accuracy yielded by the transformer neural operator, in comparison to this naive approach.
\begin{figure}[htbp]
    \centering
    \includegraphics[width=0.6\linewidth]{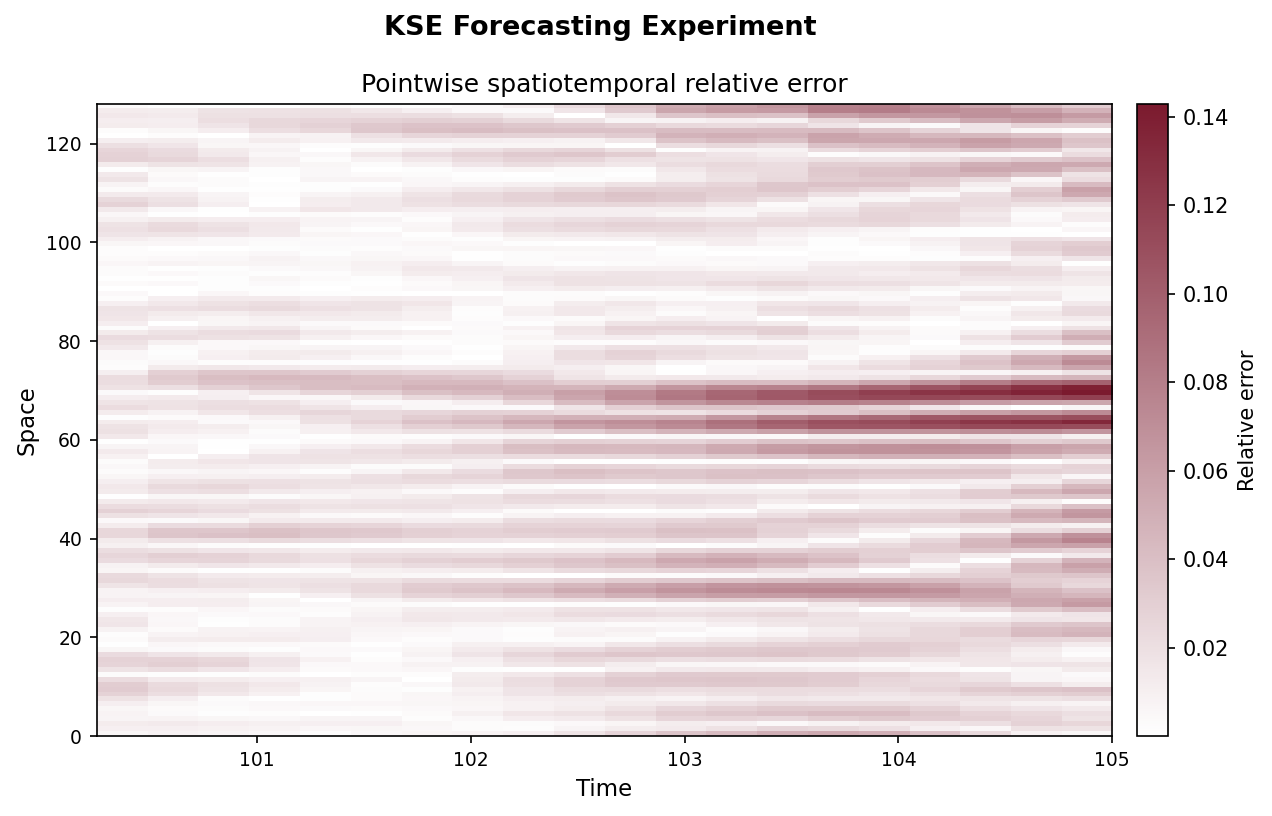}
    \caption{KS. Pointwise errors relative to the root mean square of the ground truth (computed in space) for the sample in the test set achieving the median relative $L^2$ error.}
    \label{fig:kse_spatiotemporal}
\end{figure}

\begin{figure}[htbp]
    \centering
    \begin{subfigure}{0.48\linewidth}
        \centering
        \includegraphics[width=\linewidth]{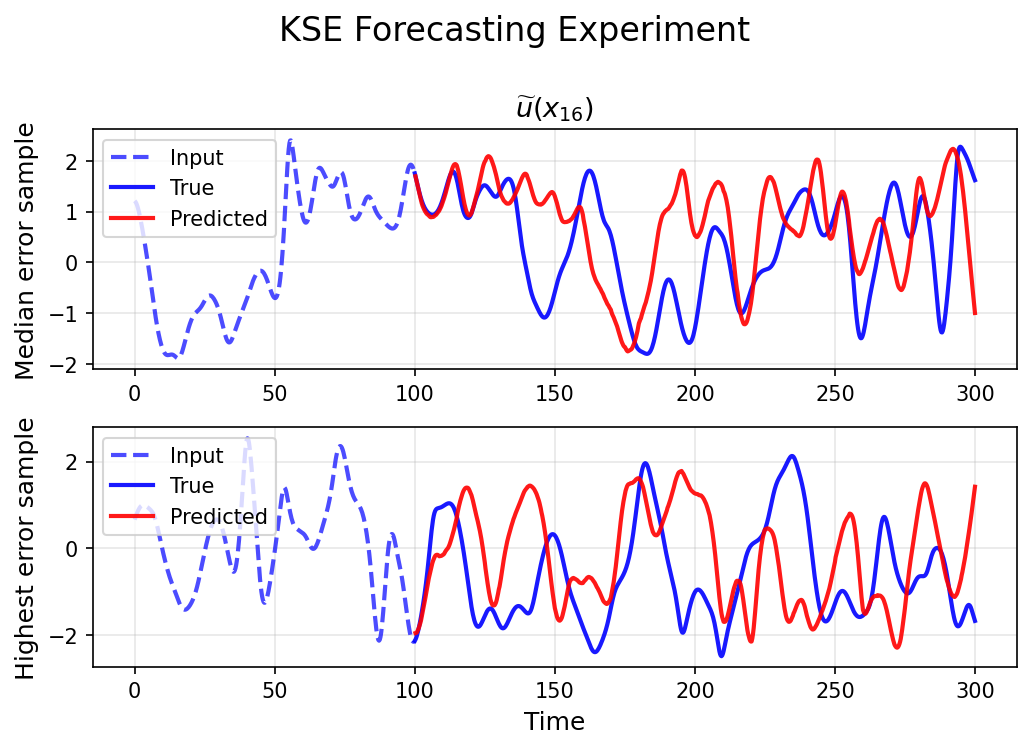}
        \caption{}
        \label{fig:kse_forecasting}
    \end{subfigure}\hfill
    \begin{subfigure}{0.48\linewidth}
        \centering
        \includegraphics[width=\linewidth]{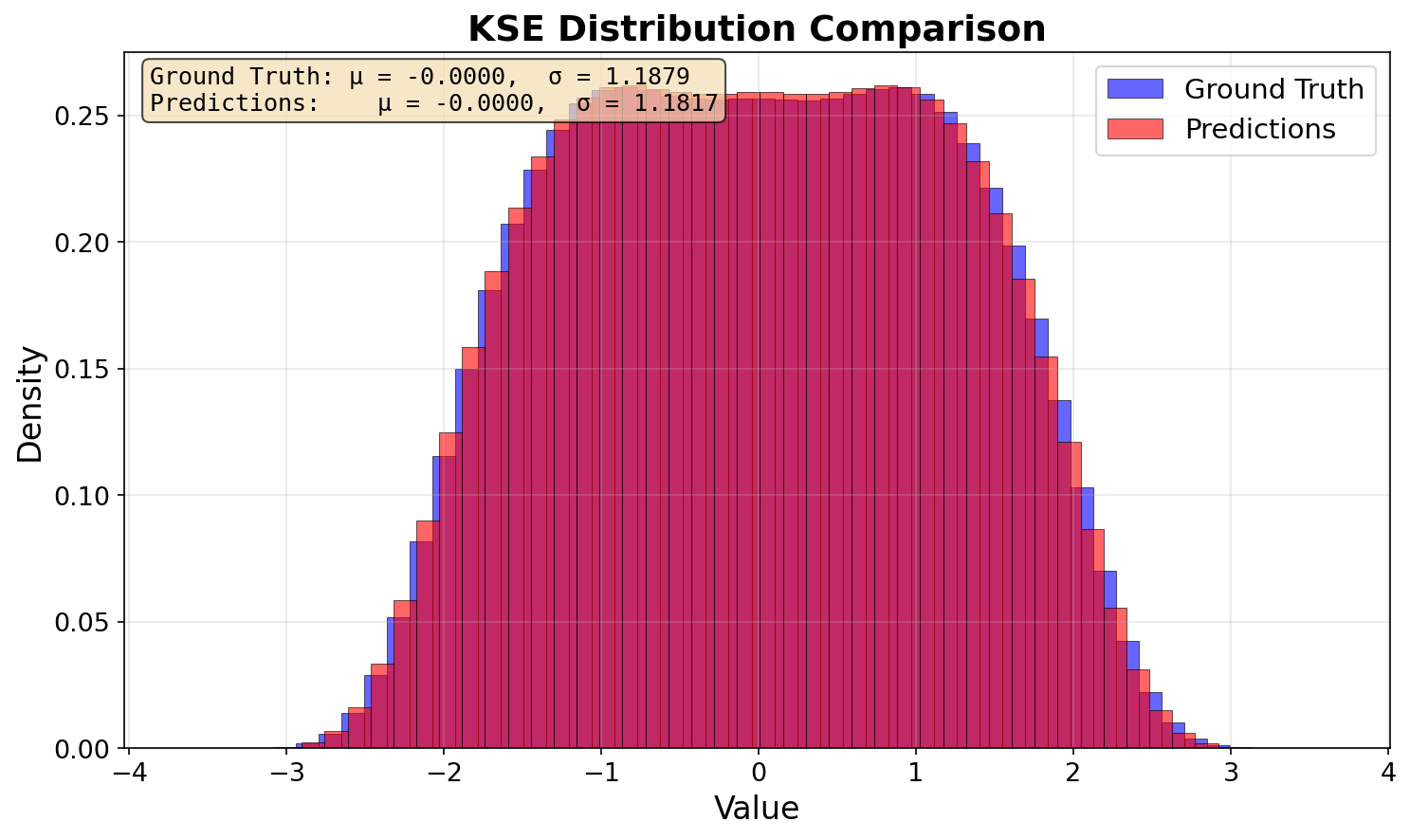}
        \caption{}
        \label{fig:kse_forecasting_hist}
    \end{subfigure}
    \caption{KS. Left: forecast obtained by composing the learned map on one-step median and worst-case relative $L^2$ error samples. Right: distribution of trajectory predictions under composition of the learned forecasting map; the dark shade of red represents overlap between the blue ground truth distribution and the light red predicted distribution. The prevalence of the dark shade of red thus indicates matching of distributions.}
    \label{fig:kse_comparison}
\end{figure}

\section{Conclusions and Future Directions}
\label{sec:conclusions}

The theory developed in this paper provides a mathematical framework
for the evaluation and development of data-driven methods for 
forecasting. The methods are model agnostic, avoiding the need for specified dynamics and possibly expensive model evaluations, and the ideas transfer
to the smoothing problem.  The framework provides, for example, an idealized analogue of the direct-observation state estimators and direct-observation forecasts that are currently being explored in AI-weather forecasting; see \cite{Allen2025} and \cite{alexe2024graphdop}, respectively. These methods learn maps from observations to the state, or to future observations, which is the structure of the smoothing and forecasting maps studied here. However, they operate with noisy observations, available at discrete times and irregularly in space, whereas our analysis is in the noise-free, continuous-time setting; our results establish the existence and approximability of such maps only in this idealized setting, and extending them to the settings above is discussed in the directions for future work below.
Indeed, many open challenges remain, 
and we highlight several interesting avenues for future work.
\begin{enumerate}
    \vspace{1ex}
\item Developing similar ideas to those in this paper, but allowing
for stochastic effects, would be valuable. Rather than leading to
map $\Psimd: C([0,T];\R^{d_p}) \to C([0,T];\R^{d_p})$ the goal
would be to define, and approximate a Markov process on 
$C([0,T];\R^{d_p}).$
\vspace{1ex} 
\item Developing similar ideas to those in this paper, but in discrete
time, would be valuable. Rather than leading to
map $\Psimd: C([0,T];\R^{d_p}) \to C([0,T];\R^{d_p})$ this will lead to
$\Psimd: \R^{nd_p} \to \R^{nd_p}$, where $n$ is the number of lags;
exploring connections to Takens'
embedding  \cite{takens2006detecting,sauer1991embedology} will be natural.
\vspace{1ex}
    \item In \Cref{ex:Lorenz} we show that the observability-rank condition is satisfied by the Lorenz `63 dynamical system whenever $x_0\neq 0$. Characterizing for which points the Lorenz `96 dynamical system and Kuramoto-Sivashinsky equation satisfy this condition constitutes an avenue for further work. A similar analysis for the Navier-Stokes equation could lead to theory directly applicable to computational fluid dynamics and the atmospheric sciences.
    \vspace{1ex}
    \item In this paper we have not exploited any specific structural properties
    of the dynamics; but doing so might lead to more detailed understanding
    of the maps defining the smoothing and forecasting problems. For example
    the property of synchronization \cite{pecora1990synchronization} might be
    exploited, or the existence 
of an inertial manifold \cite{temam2012infinite}.
    \vspace{1ex}
    \item As mentioned in \Cref{rem:global}, it is of interest to extend the local existence results in this paper to show under a different set of assumptions the existence of an operator mapping observations to unobserved variables, defined on the whole domain of the dynamics.
\vspace{1ex}
\vspace{1ex}
\item In \eqref{eq:dynamical_system_intro} part of the state is observed directly; as explained in \Cref{appendix:observability_control}, this is the special case $h=\pi_p$ of the standard setting of \cite{hermann1977nonlinear}. When the components of a general observation function $h$ are independent, a change of coordinates in which $h(x)$ forms part of the state brings the general setting into the form \eqref{eq:dynamical_system_intro}. Extending the analysis of this paper to general observation functions would be valuable.
\vspace{1ex}

\end{enumerate}

\appendix

\section{Observability-Rank Condition in Control Theory} 
\label{appendix:observability_control}

In the following discussion we explain the connection between the observability-rank condition, employed in this paper, and the observability-rank condition from control theory, introduced in the seminal paper \cite{hermann1977nonlinear}. In the general control-theoretic context, consider the partially observed dynamical system
\begin{align}
\label{eq:hermann_ds}
    \begin{aligned}
        \dot{x} &= \phi(x, u),\\
        y &= h(x),
    \end{aligned}
\end{align}
where $u(t) \in \Omega\subset\R^l$ for each $t \in \R^+$, $x \in M$, a $C^{\infty}$ connected manifold of dimension $m$, $y \in \R^n$, and $\phi,h$ are $C^{\infty}$ functions. For simplicity, in what follows we
consider the setting where $M=\R^{d_p + d_q}$ and $h: \R^{d_p+d_q} \to \R^{d_p}$. The following links the main body of the paper to this setting.

\begin{example} \label{ex:ex1}
The partially observed dynamical system \eqref{eq:dynamical_system_intro} we consider in this paper may be written in the form
\begin{align}
\label{eq:hermann_ds_ours}
    \begin{aligned}
        \dot{p} &= f(p,q),\\
        \dot{q} &= g(p,q),\\
        p &= \pi_p(p,q),
    \end{aligned}
\end{align}
This is a specific example of the control theoretic setting with $M=\R^{d_p + d_q}$, the external control $u$ absent, $x=(p,q)$, $\phi=(f,g)$ and $h = \pi_p$.
\end{example}

We now introduce the technical notions required to define the control theoretic
version of the observability-rank condition from \cite{hermann1977nonlinear}; these definitions may be found 
in \cite{hermann1977nonlinear} and are developed for general $C^\infty$ manifolds $M$ and general $C^\infty$ observation operators $h:\R^m\to \R^n$. Here we work
in the setting of $M=\R^{d_p + d_q}$ and $h:\R^{d_p+d_q}\to \R^{d_p}$, aligning with the Example \ref{ex:ex1},
and in particular with the contents of our paper.

Every point $\mathsf{u} \in \Omega$ defines a vector field $x\mapsto \phi(x, \mathsf{u}) \in C^\infty(\R^{d_p+d_q};\R^{d_p+d_q})$; we let $\mathcal{F}^{0}\subset C^\infty(\R^{d_p+d_q};\R^{d_p+d_q})$ denote the collection of all of these vector fields, one for each $\mathsf{u} \in \Omega$.
We let $\mathcal{H}^{0}$ denote the subset of $C^\infty(\R^{d_p+d_q})$ consisting of the $d_p$ component functions $h_1, h_2, \ldots, h_{d_p}:\R^{d_p+d_q} \to \R$ of the observation map $h:\R^{d_p+d_q}\to \R^{d_p}$, viewed as scalar-valued smooth functions on $\R^{d_p+d_q}$. We define $\mathcal{H}$ to be the smallest linear subspace of $C^\infty(\R^{d_p+d_q})$ containing $\mathcal{H}^{0}$ and which is closed with respect to Lie differentiation {under} elements of $\mathcal{F}^{0}$. For each $\varphi \in \mathcal{H}$, the differential $\sD\varphi(x) : \mathbb{R}^{d_p+d_q} \to \mathbb{R}$ is a linear functional on $\mathbb{R}^{d_p+d_q}$, i.e.\ an element of $(\mathbb{R}^{d_p+d_q})^*$. We define $\mathrm{d}\mathcal{H}^0 = \{\sD\varphi : \varphi \in \mathcal{H}^0\}$ and $\mathrm{d}\mathcal{H} = \{\sD\varphi : \varphi \in \mathcal{H}\}$ as collections of maps $\mathbb{R}^{d_p+d_q} \to (\mathbb{R}^{d_p+d_q})^*$, and $\mathrm{d}\mathcal{H}(x) \subseteq (\mathbb{R}^{d_p+d_q})^*$ as the subspace obtained by evaluating elements of $\mathrm{d}\mathcal{H}$ at $x$.

\begin{definition}[Observability-rank condition in \cite{hermann1977nonlinear}]
    The dynamical system \eqref{eq:hermann_ds} satisfies the \textit{observability-rank condition} at $\fx$ if $\dim(\rd\mathcal{H}(\fx)) = d_p+d_q$.
\end{definition}

With the translation between our paper and \cite{hermann1977nonlinear}, as described in Example \ref{ex:ex1}, we may formulate the following proposition.

\begin{proposition}
\label{prop:connection_control}
    The dynamical system \eqref{eq:dynamical_system_intro} satisfying \Cref{assump:regularity} with $f,g\in C^\infty$ and satisfying \Cref{assump:observability} at the point $(\fp,\fq)$ is an instance of a dynamical system \eqref{eq:hermann_ds} with a fixed control, satisfying the observability-rank condition from \cite{hermann1977nonlinear} at the point $\fx=(\fp,\fq)$.
\end{proposition}

\begin{proof}[Proof of \Cref{prop:connection_control}]
We are in the specific setting of Example \ref{ex:ex1} and hence, because no control is present,
the set $\mathcal{F}^0$ simply comprises the vector field $(f,g)\in C^\infty(\R^{d_p+d_q};\R^{d_p+d_q})$. 
Letting $\fx = (\fp, \fq)$ and recalling the notation for Lie derivatives defined in \Cref{subsec:notation}, we have that
\begin{subequations}
\label{eq:LiePi}
\begin{align}
    \mathcal{L}_{(f,g)}(\pi_p)(\fx) &= \frac{\partial \pi_p}{\partial x}(\fx)\begin{pmatrix}
        f(\fx)\\
        g(\fx)
    \end{pmatrix}\\
    &= \begin{pmatrix}
        I_{d_p \times d_p} & \mathbf{0}_{d_p \times d_q}
    \end{pmatrix}\begin{pmatrix}
        f(\fp, \fq)\\
        g(\fp, \fq)
    \end{pmatrix} = f(\fp, \fq).
\end{align}
\end{subequations}
Therefore, higher order Lie derivatives of $\pi_p$ under the vector field $(f,g)$ yield $\cL^i_{(f,g)}\, f$. Let $\mathcal{G}$ be the span of the component functions 
$\{\pi_{p,1}, \ldots, \pi_{p,d_p}\} \cup \{\mathcal{L}^i_{(f,g)} f_j : 1 \leq j \leq d_p,\, 0 \leq i \leq n-1\}$,
where $\pi_{p,j} : \mathbb{R}^{d_p+d_q} \to \mathbb{R}$ denotes the $j$th component of $\pi_p$ and similarly $f_j$ denotes the $j$th component of $f$, so that the evaluations $\{\pi_{p,j}(\fp,\fq)\}_{j=1}^{d_p}$ and $\{\mathcal{L}^i_{(f,g)} f_j(\fp,\fq)\}_{j=1}^{d_p}$ correspond to the $(n+1)d_p$ entries of $F^{(n)}(\fp,\fq)$ in \eqref{eq:Lie}.

From \eqref{eq:LiePi} and closure of $\mathcal{H}$ under Lie differentiation under $(f,g)$, we have that $\cG \subseteq \mathcal{H}$; consequently, $\rd\mathcal{H}$ contains differentials of elements of $\cG$. The linear transformation $L:\R^{(n+1)d_p}\to \R^{d_p+d_q}$ from \eqref{eq:observability_eq_inv2} defines a map 
\begin{align}
    LF^{(n)} \colon (\fp, \fq) \mapsto [s_1(\fp, \fq), s_2(\fp,\fq), \ldots, s_{d_p + d_q}(\fp, \fq)]^{\top},
\end{align}
for $s_i \in \cG$. Since $\mathrm{d}\mathcal{H}(\fp,\fq) \subseteq (\mathbb{R}^{d_p+d_q})^*$, we have that $\dim(\rd\mathcal{H}(\fp, \fq)) \leq d_p+d_q$. Since $s_1, \ldots, s_{d_p+d_q} \in \mathcal{G} \subseteq \mathcal{H}$, their differentials 
$\sD s_1(\fp,\fq), \ldots, \sD s_{d_p+d_q}(\fp,\fq)$ are elements of $\mathrm{d}\mathcal{H}(\fp,\fq)$. 
Moreover, these differentials are the rows of the Jacobian $\sD(LF^{(n)})(\fp,\fq)$, 
which has rank $d_p+d_q$ by \Cref{assump:observability}, and hence are linearly independent. 
Therefore $\dim(\mathrm{d}\mathcal{H}(\fp,\fq)) \geq d_p+d_q$, and combined with the upper bound 
$\dim(\mathrm{d}\mathcal{H}(\fp,\fq)) \leq d_p+d_q$, we conclude that $\dim(\mathrm{d}\mathcal{H}(\fp,\fq)) = d_p+d_q$.
\end{proof}

\section{Proofs of Smoothing Theorems}
\label{sec:PST}

\begin{proof}[Proof of Theorem \ref{prop:psiS}]
We construct such a $\Psi^\dagger$ using $W$ from \Cref{prop:W} and composition with the solution operator $\Phi$. Define $\Psi_S^\dagger : S^I_{[0,T],p} \subset C^k([0,T];\R^{d_p}) \to S^I_{[0,T],q} \subset C([0, T]; \R^{d_q})$ by
\begin{equation}
    \label{eq:p_to_q_map2}
    \Psi_S^\dagger(p)(t) = \pi_q \circ \Phi(t, W(p)), \quad 0\leq t\leq T.
\end{equation} 
Recall that $\pi_q$ is continuous. By construction $W(p) = (p(0), q(0))$, so $\pi_q \circ \Phi(t, W(p)) = q(t)$. Observe that this is a continuous operator since $\Phi$ is continuous with respect to initial conditions, $W$ is continuous with respect to $p$ and composition of continuous operators is continuous.
\end{proof}

\begin{proof} [Proof of Theorem \ref{thm:UAsmoothing}]
We know that $\Psi_S^{\dagger}$ exists and continuous since the assumptions of \Cref{prop:psiS} are satisfied. We recall that $S^I_{[0,T],p}$ is compact. The conclusion follows from a direct application of \Cref{thm:UANeuralOperator} with $D = [0, T]$, $r = d_p$, $r' = d_q$, $\Psi^{\dagger} = \Psi_S^{\dagger}$, and $K = S_{[0,T],p}$.
\end{proof}

\section{Training and Architectural Details}
\label{appendix:architectures}

In this section we describe the transformer neural operator architectures employed in \Cref{sec:experiments},
highlighting the use of self- and cross-attention. The goal of this section is three-fold: firstly
to present the training details in \Cref{ssec:TD}; secondly to present the architectural details in \Cref{ssec:AD};
and thirdly to discuss how the architecture may be cast in the universal approximation context of \Cref{thm:UANeuralOperator}, in \Cref{ssec:AUA}. Recall that \Cref{thm:UANeuralOperator}  is used in the proof of \Cref{thm:UAsmoothing,thm:UAforecasting} which underpin the methodology developed and deployed in this paper. The  transformer neural operator used for the smoothing problem is based on self-attention; that used for forecasting
employs a combination of self-attention and cross-attention. 

\subsection{Training Details} \label{ssec:TD}

The set $\Theta$ is assumed to be a finite dimensional parameter space from which a $\theta^*\in\Theta$ is selected so that $\Psi_{\placeholderB}(\placeholder,\theta^*) \approx \Psi_{\placeholderB}^\dagger$. (Here $\placeholderB$ is either $S$ or $F$.) To define
$\theta^*$ we employ a cost functional $c$ and solve the following optimization problem:
\begin{equation}
\label{eq:data_assimilation_model}
\theta^*={\arg \min}_{\theta\in\Theta}\mathbb{E}_{\placeholder}\Bigl[ c\bigl(\Psi_{\placeholderB}(\placeholder,\theta), \Psi_{\placeholderB}^\dagger(\placeholder) \bigr)\Bigr].
\end{equation}
In practice, we will only have access to the values of each input and output trajectory at a set of discretization points of the domain, which we denote as $\{t_i\}_{i=1}^N$. The self-attention based transformer neural operator that we deploy for smoothing uses the same number of grid-points for the input and output; employing a cross-attention based transformer neural operator is used to bypass this constraint on matching input and output grids for forecasting in some settings, where the input is on the time-interval $[0,T]$ and the output on $[T,T+\tau]$. We note that due to its discretization invariant properties, this architecture can be used to obtain a prediction for the output at any arbitrary discretization of $[T,T+\tau]$. For the smoothing experiments we use $6$ self-attention transformer neural operator layers with $128$ latent channels, for each self-attention operator, we use $8$ attention heads. For the forecasting experiments we use an architecture composed of an encoder and a decoder: the encoder is comprised of $6$ self-attention transformer neural operator layers with $128$ latent channels for the Lorenz `63 experiment and $1024$ latent channels for the others; the decoder on the other hand is comprised of a single layer consisting of a self-attention operator and a cross-attention operator both with $128$ latent channels for the Lorenz `63 experiment and $1024$ latent channels for the others.

\subsection{Architectural Details} \label{ssec:AD}

The use of cross-attention enables the output of the architecture to be defined on a different number of grid points than the input. The self-attention based transformer neural operator we use is described in detail in \cite[Section 4.2]{calvello2024continuum}. A slight modification to the architecture implemented in practice leads to a universal approximation theorem of the form we apply in this paper (as in \Cref{thm:UANeuralOperator}), the statement of which may be found in \cite[Theorem 22]{calvello2024continuum}. We focus the following discussion on the cross-attention based transformer neural operator. To this end, we recall the definition of cross-attention in function space from \cite{calvello2024continuum}.

Let $D \subseteq \R^d$ and $E \subseteq \R^e$ be open sets. Let $u: D \to \R^{d_u}$ be a function,
$v: E \to \R^{d_v}$ another function, and $x \in E$, $y \in D$ points. Then, for learnable parameters $Q\in \R^{d_K\times d_v},K\in \R^{d_K\times d_u},V\in \R^{d_V\times d_u}$, cross-attention is defined as the operator acting on functions $v,u$ such that
\begin{equation}
\label{eq:cross_attention}
    \sC(v,u)(x) = \bbE_{y \sim \pi(y;v,u,x)} [V u(y)],
\end{equation}
for any $x\in E$, where the probability density function $\pi(\cdot\,;v,u,x): D \to \R^+$ is defined as
\begin{equation}
\label{eq:cross_attention_pdf}
    \pi(y; v,u, x) = \frac{\mathrm{exp} \Big ( \big\langle Q v(x), K u(y) \big\rangle_{\R^{\dK}} \Big )}{\int_{D} \mathrm{exp} \Big ( \big\langle Q v(x), K u(s) \big\rangle_{\R^{\dK}} \Big ) \: \mathrm{d}s},
\end{equation}
for any $y \in {D}$. We note that self-attention is a special case of cross-attention where
$D=E$ and $v=u$. In the following we let $\cV,\cU,\cW,\cZ$ denote arbitrary function spaces; furthermore, while cross-attention allows a more general definition, here we write the architecture as a mapping between functions over the common domain $D$. The cross-attention based transformer neural operator we employ may be written as the map $\Psi_\sC(\placeholder, \placeholder;\theta):\cV(D;\R^\ell)\times \cU(D;\R^r)\to \cZ(D;\R^{r'})$ acting on functions $v\in \cV(D;\R^{\ell})$ and $u\in \cU(D;\R^r)$ such that
\begin{equation}
\label{eq:transformer_operator_compact}
\Psi_{\sC}(v, u;\theta) \coloneqq \Bigl(\sT_{\textrm{out}}\circ\sD(v, \placeholder)\circ\sE_L \circ \sT_{\textrm{in}}\Bigr) (u,\theta).
\end{equation}
In practice, we define $\sT_{\textrm{in}}:\cU(D;\R^r)\to \cW(D;\R^c)$ via concatenation with the grid (positional encoding) and application of a learnable linear transformation that lifts to a latent channel dimension $c$. The operator $\sE_L: \cW(D;\R^c)\to \cW(D;\R^c)$ is defined via a composition of $L$ self-attention layers. Its construction is outlined in detail in \cite[Section 4]{calvello2024continuum}. The transformer decoder $\sD: \cV(D;\R^\ell)\times \cW(D;\R^c)\to \cW(D;\R^c)$ is defined for $(v,u)\in \cV(D;\R^\ell)\times \cW(D;\R^c)$ via the iteration
\begin{subequations}
\label{eq:recurrence_enc_def}
\begin{align}
    v &\mapsfrom S_{\mathrm{in}}v,\\
    v &\mapsfrom {W^\sD_1}v + \sC(v, v), \\
    v &\mapsfrom \sF_{\textrm{LayerNorm}}(v),\\
    \label{eq:reason}
    v &\mapsfrom {W^\sD_2}v + \sC(v, u), \\
    v &\mapsfrom \sF_{\textrm{LayerNorm}}(v), \\
    v &\mapsfrom {W^\sD_3}v + \sF_{\textrm{NN}} (v),\\
    v &\mapsfrom \sF_{\textrm{LayerNorm}}(v),
\end{align}
\end{subequations}
where $S_{\mathrm{in}}$ is a learnable linear transformation lifting to the channel dimension, $\sF_{\textrm{LayerNorm}}$ denotes application of layer normalization, $\sF_{\textrm{NN}}$ a two layer multilayer perceptron. We note that in practice, multihead cross-attention is employed; we refer to \cite[Section 4]{calvello2024continuum} for a discussion on multihead attention. We also note that in our implementation, we fix the linear transformation $W^{\sD}_1, W^{\sD}_2, W^{\sD}_3$ to be the identity. For implementation purposes, we define the operator $\sT_{\textrm{out}}:\cW(D;\R^c)\to \cZ(D;\R^{r'})$ as a linear transformation projecting to the output space.

Due to the definition of cross-attention and step \eqref{eq:reason} in the decoder, the discretization size of the output function corresponds to the discretization size of the input function $v$, making this architecture particularly attractive for forecasting problems, where the output time interval may be of length differing to the input. In practical application of this architecture, the $v$ is fixed to be the rescaled identity function, sampled at the grid points where the output function is to be evaluated. Since the $v$ is fixed, for the approximation theory developed in the following section we only require consideration of the operator $\Psi(v,\placeholder;\theta): \cU(D;\R^r)\to \cZ(D;\R^{r'})$. Because the architecture is the discrete version of the continuum neural operator, the scheme is discretization invariant in both inputs, meaning that the parameters will be independent of any discretization of input and output functions; the resulting scheme is thus deployable at any discretization. In the next subsection we state and prove a universal approximation theorem for a slight variant of the scheme implemented in practice.

\subsection{Universal Approximation for Cross-Attention} \label{ssec:AUA}

Consider activation functions $\sigma \in C^\infty(\R)$ which are non-polynomial and Lipschitz continuous. We consider neural operators of the form \eqref{eq:transformer_operator_compact} where $\sT_{\textrm{in}}: \cU(D;\R^r)\to \cW(D;\R^c)$ and $\sT_{\textrm{out}}: \cW\bigl(D;\R^c\bigr)\to\cZ\bigl(D;\R^{d_z} \bigr)$ are defined by neural networks of the form 
\begin{align}
    \Bigl(\sT_{\textrm{in}}(u)\Bigr)(x) &= R_2\sigma\Bigl(R_1\bigl(u(x),x\bigr)+b_R \Bigr)+b_R',\\
    \Bigl(\sT_{\textrm{out}}(v)\Bigr)(x) &= P_2\sigma\Bigl(P_1\bigl(v(x),x\bigr)+b_P \Bigr)+b_P',
\end{align}
where $\bigl(u(x),x\bigr)\in \R^{r+d}$ and $\bigl(v(x),x\bigr)\in \R^{c+d}$, where $R_1,R_2,P_1,P_2$ are learned linear transformations of appropriate dimensions and $b_R,b_R',b_P,b_P'$ are learned vectors. We define the operator $\sE: \cW\bigl(D;\R^c \bigr) \to \cW\bigl(D;\R^c \bigr)$ as a two-step map acting on its inputs $u\in\cW(D;\R^c)$ as
\begin{subequations}
\label{eq:recurrence_enc_func_ua}
\begin{align}
    {u}(x) &\mapsfrom W_1^\sE{u}(x) + \sC\big({u},u; Q^\sE,K^\sE,V^\sE\big)(x), \\
    u(x) &\mapsfrom W_2^\sE{u}(x) + W_3^\sE \sigma\bigl(W_4^\sE{u}(x)+b_1^\sE\bigr) + b_2^\sE,
\end{align}
\end{subequations}
for any $x\in D$. Note that the layer thus defined does not include layer normalizations, and hence is a variant of the self-attention transformer layer employed in practice. We also define a no-layer normalization variant of the decoder in \eqref{eq:recurrence_enc_def}, which also includes a residual of the output of the encoder after application of cross-attention; indeed we consider $\sD: \cV(D;\R^\ell)\times \cW(D;\R^c)\to \cW(D;\R^c)$ defined for $(v,u)\in \cV(D;\R^\ell)\times \cW(D;\R^c)$ via the iteration
\begin{subequations}
\label{eq:recurrence_enc_def1}
\begin{align}
    v(x) &\mapsfrom S_{\mathrm{in}}v(x)\\
    v(x) &\mapsfrom {W^\sD_1}v(x) + \sC(v, v; Q^\sD_1,K^\sD_1,V^\sD_1)(x), \\
    v(x) &\mapsfrom {W^\sD_2}\bigl(v(x),u(x)\bigr) + \sC(v, u; Q^\sD_2,K^\sD_2,V^\sD_2)(x), \\
    v(x) &\mapsfrom {W^\sD_3}v(x) + W_4^\sD \sigma\bigl(W_5^\sD{v}(x)+b_1^\sD\bigr)+ b_2^\sD,
\end{align}
\end{subequations}
for any $x\in D$. We may now apply the result of \cite{lanthaler2023nonlocal} to show two universal approximation theorems for the resulting cross-attention based transformer neural operator.

\begin{theorem}
\label{thm:ua1}
Let $D\subset \R^d$ be a bounded domain with Lipschitz boundary, and fix integers $s,s'\geq 0$. If $\Psi^\dagger:C^s\bigl(\Bar{D};\R^r\bigr) \to C^{s'}\bigl(\Bar{D};\R^{r'}\bigr)$
is a continuous operator and $K\subset C^s\bigl(\Bar{D};\R^r\bigr)$ a compact set, then for any $\epsilon>0$, there exists a cross-attention based transformer neural operator 
$\Psi(v,\placeholder;\theta):K\subset C^s\bigl(\Bar{D};\R^r\bigr) \to C^{s'}\bigl(\Bar{D};\R^{r'}\bigr)$ so that
\begin{equation}
\label{eq:ua_condition}
\sup_{u\in K} \left\|\Psi^\dagger(u) - \Psi(v,u;\theta)   \right\|_{C^{s'}} \leq \epsilon.
\end{equation}
\end{theorem}

\begin{proof}
We begin by noting that for $Q^\sD_2,K^\sD_2=0$ and $V^\sD_2 = I$, the cross-attention mapping employed in the decoder reduces to 
\begin{equation}
\sC\bigl(v,u;Q^\sD_2,K^\sD_2,V^\sD_2 \bigr)(\placeholder) = \mathbb{E}_{y\sim \pi(y;v,u,\placeholder)}[V^\sD_2 u(y)] = \frac{1}{|D|}\int u(x)~\mathrm{d} x.
\end{equation}
For encoder weights $V^\sD=0$, $W_3^\sE=0$, $W_1^\sE=W_2^\sE=I$ and decoder weights $W^\sD_2=\bigl(0, W\bigr)$, $W^\sD_3 = 0$, $W^\sD_4=W^\sD_5=I$ and $b^\sD_2 = 0$, the composition of encoder \eqref{eq:recurrence_enc_func_ua} and decoder \eqref{eq:recurrence_enc_def1} reduces to the mapping 
\begin{equation}
\label{eq:nonlocal_NO}
    u(\placeholder) \mapsto \sigma\left(Wu(\placeholder) + b_1 + \frac{1}{|D|}\int u(x)~\mathrm{d} x \right).
\end{equation}
The existence of $W,R_1,R_2,P_1,P_2,b_1,b_R,b_R',b_P,b_P'$ so that $\Psi(v,\placeholder\,;\theta)$ satisfies \eqref{eq:ua_condition} then follows from \cite[Theorem 2.1]{lanthaler2023nonlocal}, which also involves the application of the universality result for two-layer neural networks of \cite[Theorem 4.1]{Pinkus_1999}.
\end{proof}

\section*{Acknowledgments}

AMS is supported by a Department of Defense (DoD) Vannevar Bush Faculty Fellowship (award N00014-22-1-2790).
Both EC's acknowledges support from the Vannevar Bush Faculty Fellowship held by AS; Calvello also
acknowledges support from the Resnick Sustainability Institute and Carlson the Caltech Von Karman Instructorship. The computations presented in this paper were partially conducted using resources from the Resnick High Performance Computing Center, a facility supported by the Resnick Sustainability Institute at the California Institute of Technology, and resources from the National Energy Research Scientific Computing Center (NERSC), a Department of Energy User Facility.

\bibliographystyle{siamplain}
\bibliography{references}

\end{document}